\documentclass[preprint,12pt,3p,authoryear,nopreprintline]{elsarticle}

\usepackage{amssymb}
\usepackage{amsmath}
\usepackage{amsthm}

\usepackage{lineno}

\usepackage{mathtools}
\usepackage{multirow}
\usepackage{microtype}
\usepackage{graphicx}
\usepackage{booktabs}

\usepackage{algorithm}
\usepackage{algorithmic}

\usepackage[table]{xcolor}

\usepackage{subcaption}

\theoremstyle{plain}
\newtheorem{theorem}{Theorem} 

\allowdisplaybreaks[4]

\usepackage{hyperref}

\begin{document}

\begin{frontmatter}

\title{FES-FM: Free Energy Surface Sampling via Reduced Flow Matching} 

\author[a]{Zichen Liu}
\author[a,d,e]{Tiejun Li \corref{cor}}
\affiliation[a]{organization={Center for Data Science, Peking University}}
\affiliation[d]{organization={LMAM and School of Mathematical Sciences, Peking University}}
\affiliation[e]{organization={Center for Machine Learning Research, Peking University}}
\cortext[cor]{Corresponding author. E-mail address: tieli@pku.edu.cn}

\begin{abstract}
Sampling the distribution of collective variables (CVs) and estimating the associated free energy surface are crucial problems in statistical physics, as they underpin a better understanding of chemical reactions and conformational transitions. Traditional methods usually rely on simulations in high-dimensional configuration space and project the resulting configurations onto the CV space. To improve sampling speed, we propose FES-FM, a reduced flow matching (FM) method for free energy surface (FES) sampling. We train a dynamical transport map in the CV space, thereby enabling direct sampling of CV distributions and reconstruction of the corresponding free energy surface. For many-particle systems, we construct a prior distribution based on the Hessian at a local minimum of the potential, which ensures both rotation-translation invariance and physically meaningful configurations. We evaluate the proposed method across a variety of potential functions and collective variables, including alanine dipeptide in implicit solvent as a molecular benchmark. Comparative experiments demonstrate that our approach significantly improves sampling speed while maintaining accuracy.
\end{abstract}

\begin{keyword}
  free energy surface \sep flow matching \sep Boltzmann sampling \sep annealed importance sampling
  \end{keyword}
  \end{frontmatter}
  
\section{Introduction}

Boltzmann sampling is a central task in statistical mechanics, where the goal is to sample from a target probability density defined as $p(x) = Z^{-1} e^{-U_{\mathrm{target}}(x)}$. Here, $U_{\mathrm{target}}(x)$ represents an unnormalized potential energy, and $ Z = \int_{\mathbb{R}^n} e^{-U_{\mathrm{target}}(x)} \mathrm{d}x $ is a normalization constant (partition function) that is typically intractable to compute. For simplicity, we absorb the inverse temperature into $U_{\mathrm{target}}$.

In many molecular systems, the quantities of interest are low-dimensional descriptors known as collective variables (CVs). CVs are physically meaningful functions of the high-dimensional configuration variables and are designed to capture the progress of rare events such as chemical reactions and conformational transitions. As a result, an important goal beyond sampling configurations is to characterize the distribution of CVs and the associated free energy surface.

The CV map is denoted by $\xi:\mathbb{R}^n\to\mathbb{R}^d$ with $d\ll n$, which projects a high-dimensional configuration into the low-dimensional CV space $\mathbb{R}^d$. Our objective is, given the mapping $\xi(x)$, to construct a sampler for the induced CV random variable $Y=\xi(X)$ where $X\sim p(x)$. Assuming $\xi$ is sufficiently regular (e.g., $\xi\in C^1$), $\nabla \xi\in \mathbb{R}^{n\times d}$ and $\operatorname{rank}(\nabla \xi(x))=d$ in the region of interest, the density $\rho(y)$ of $Y$ can be expressed as
\begin{align}
\rho(y)=& \int p(x) \delta(\xi(x)-y) \mathrm{d}x = \int_{\Sigma_{y}} p(x) |J_{\xi}(x)|^{-1} \mathrm{d} \sigma_{\Sigma_{y}}(x).
\end{align}
Here, $|J_{\xi}(x)|=\det(\nabla \xi(x)^{T} \nabla \xi(x))^{\frac{1}{2}}$ denotes the normal Jacobian in the coarea formula. $\Sigma_y = \{ x \in \mathbb{R}^n : \xi(x) = y \} $ is the $y$-level set of $\xi$, and $ \sigma_{\Sigma_y} $ is the surface measure on $ \Sigma_y $ induced by the Euclidean metric on $\mathbb{R}^n$. The free energy is defined as $F(y) = -\log \rho(y)$ and $F(y)$ (or even its unnormalized counterpart) is generally intractable to evaluate. \textbf{Our goal is to directly generate samples $Y\sim \rho(y)$ in the CV space.}

Traditional numerical simulation methods estimate $\rho(y)$ by simulating long trajectories in the original configuration space, or in CV space through analytically derived reduced dynamics \cite{stoltz2010free, frenkel2023understanding}, which can be expensive in rare-event regimes and for high-dimensional systems. To improve sampling speed, we propose \textbf{FES-FM}, a reduced flow matching method for free energy sampling. We learn a dynamical transport in the CV space that pushes a simple prior distribution toward the target CV distribution $\rho(y)$. Accordingly, we derive the training objective based on the transport equation satisfied by the transport velocity and combine it with Jarzynski-type non-equilibrium reweighting. Once trained, FES-FM generates CV samples by evolving the reduced dynamics, avoiding full-space simulation during generation and thereby accelerating sampling while maintaining accuracy. Our contributions are summarized as follows.

\begin{itemize}
\item We introduce a reduced-space flow matching framework that learns a transport map in CV space to sample the free energy surface directly, avoiding full-space simulation during generation.
\item For many-particle systems, we propose a Hessian-informed harmonic prior distribution, whose samples vibrate near a local minimum of the potential energy. This distribution not only possesses $\mathrm{E}(3)$-invariance but also ensures that its samples correspond to physically meaningful configurations.
\item We evaluate FES-FM on a suite of benchmark potentials. We compare against full-space generative baselines that generate the Boltzmann distribution and obtain samples via the CV map. Our method significantly improves sampling speed while maintaining accuracy.
\end{itemize}

\section{Related work}

\subsection{Machine-learning methods for Boltzmann sampling} 

While our goal is to sample the free energy surface in the CV space, we summarize prior work focusing on Boltzmann sampling in the configuration space.

Several recent works develop Boltzmann samplers by constructing a transport map for the target distribution, combined with the Physics-Informed Neural Network (PINN, \citet{raissi2019physics}) framework.
\citet{chemseddine2024neural} propose an interpolation strategy that only parametrizes the potential function while fixing an appropriate velocity field. \citet{sun2024dynamical} consider both deterministic and stochastic dynamical transport maps, and study the optimal transport map via a learnable potential interpolation. 
Non-Equilibrium Transport Sampler \citep{albergo2024nets} leverages non-equilibrium sampling to construct training objectives and generates samples for optimization.

Normalizing flow \citep{noe2019boltzmann} and continuous normalizing flow \citep{zhang2018monge} are trained via maximum likelihood to match the target Boltzmann distribution instead of the intermediate processes. \citet{zhang2022path,vargas2023denoising,richter2024improved,pmlr-v267-havens25a,chen2024sequential,berner2022optimal} formulate learning diffusion processes as stochastic optimal control.
Furthermore, Adjoint Matching \citep{domingo-enrich2025adjoint,liu2025adjoint} casts stochastic optimal control problems as regression problems.
\citet{Phillips2024ParticleDD} and \citet{Bortoli2024TargetSM} also show that regression-based objectives can be effective, whereas \citet{AkhoundSadegh2024IteratedDE} and \citet{woo2024iteratedenergybasedflowmatching} introduce an offline approach that learns the score using samples stored in a replay buffer.
Energy-based models \citep{Plainer2025ConsistentSA,wang2025energy} are widely used tools for sampling, while the Jarzynski equality is also used in energy-based models \citep{Carbone2023EfficientTO} and the Helmholtz free energy calculation \citep{he2025feat}.

\subsection{Classical methods for free energy surface sampling} 

Umbrella sampling \citep{torrie1977nonphysical} applies biases in separate windows to obtain the globally unbiased distribution of the CVs. Metadynamics \citep{laio2002escaping} directly compensates free energy barriers with accumulated Gaussian biases, and uses the accumulated bias to produce the free energy surface. \citet{maragliano2006temperature} construct an extended system where CVs are treated as dynamical variables to sample the free energy surface. 
Adaptive Biasing Force \citep{comer2015adaptive} adaptively estimates the mean force to build a flattening bias along CVs, whereas Variational Enhanced Sampling \citep{valsson2014variational} optimizes a bias potential variationally to enforce a target CV distribution and recover the free energy surface.
These classical strategies typically rely on simulations in the configuration space. In light of this, we aim to train a model that requires only the simulation of the reduced dynamics in the CV space.

For additional background on free energy computation and Boltzmann sampling, see \citet{stoltz2010free} and \citet{frenkel2023understanding}.

\section{Background and preliminaries}\label{Background}

In this section, we first review flow-matching Boltzmann samplers \citep{sun2024dynamical,chemseddine2024neural}, and then present the Non-Equilibrium Transport Sampler (NETS, \citet{albergo2024nets}), which provides a key building block for our algorithm.

\subsection{Boltzmann sampling via flow matching}

\citet{sun2024dynamical,chemseddine2024neural} propose to sample the target distribution by learning an ordinary differential equation (ODE) based transport map from a simple prior, where the time-dependent velocity field is parameterized by a neural network and trained using PINN objectives.
We denote the density of the prior distribution as $p(x,0)=Z_0^{-1}e^{-U_0(x)}$. A common choice for the prior is the Gaussian distribution, i.e., $U_0(x) = \frac{1}{2}\|x\|^2$. For $t \in [0,1]$, we define the linear interpolation
\begin{equation}
U(x,t) = (1 - t) U_{0}(x) + t U_{\mathrm{target}}(x),
\end{equation}
and the corresponding distribution is $\mathrm{d} \nu_{t}(x) =p(x, t) \mathrm{d}x$, where
\begin{equation}
\quad p(x,t) = \frac{1}{Z(t)} e^{-U(x,t)}, \quad Z(t) = \int_{\mathbb{R}^{n}} e^{-U(x,t)} \mathrm{d}x.
\end{equation}
Other types of interpolation are considered in \citet{sun2024dynamical,chemseddine2024neural}. The goal of the flow-matching based Boltzmann sampler is to find a velocity field $b(x, t)$ such that the marginal distribution of the ODE solution:
\begin{equation}\label{ode_full}
\frac{\mathrm{d}}{\mathrm{d}t} X_t= b(X_t, t)
\end{equation}
satisfies $X_t\sim p(x, t)$ for all $t \in [0, 1]$. In particular, we have $X_1 \sim p(x)$. Thus, once such a $b(x, t)$ is obtained, we can sample from the target density $p(x)$ by solving the ODE~\eqref{ode_full}.

The velocity field $b(x, t)$ and the distribution $p(x, t)$ are related through the Liouville equation, i.e. $\partial_t p(x,t) + \nabla \cdot (p(x,t) b(x,t)) = 0$, and this equation can be reformulated in terms of the potential $U(x, t)$ as:
\begin{equation}\label{PINN_full_point}
\partial_t U(x,t) + b(x,t)\cdot \nabla U(x,t) - \nabla \cdot b(x,t) + \partial_t \log Z(t) =0.
\end{equation}
In this paper, for a function like $b(x,t)$ depending on both space and time, $\nabla$ denotes the gradient operator with respect to the spatial variables, i.e., $\nabla=(\partial_{x_1},\partial_{x_2},\dots,\partial_{x_n})^T$. To approximate $b(x, t)$ and $\partial_t \log Z(t)$, we parameterize them using neural networks $b_{\theta_0}(x, t)$ and $c_{\theta_0}(t)$, respectively, whose parameters are optimized by minimizing a loss function derived from the PINN framework. The loss function is given by:
\begin{align}
& \int_{0}^{1} \int_{\mathbb{R}^n} \big| \partial_t U(x,t) + b_{\theta_0}(x,t)\cdot \nabla U(x,t) - \nabla \cdot b_{\theta_0}(x,t) +  c_{\theta_0}(t) \big|^2 \hat{p}(x,t) \mathrm{d}x \mathrm{d}t.
\end{align}
Since the training objective is to ensure that \eqref{PINN_full_point} holds pointwise, the choice of $\hat{p}(x,t)$ can be arbitrary as long as it covers the desired approximation region. 

We remark that the \textit{flow matching} mentioned above is different from the flow matching in generative tasks \citep{lipman2023flow,liu2022flow,albergo2022building}. The method discussed here performs interpolation on the potential functions, whereas flow matching in generative tasks performs interpolation on samples. However, since both methods aim to find a flow map to match the predefined marginal distribution, we also refer to the method here as flow matching.

\subsection{Expectation estimation via the non-equilibrium state}\label{sec:expectation_estimation}
To accurately assign the statistical weight $\hat{p}(x,t)$ to target regions where the probability mass is transported by $b(x,t)$, NETS proposes setting $\hat{p}(x,t)=p(x,t)$, which places the statistical weight exactly in these critical regions. Under this setting and noting that $\mathrm{d} \nu_{t}(x) =p(x, t) \mathrm{d}x$, the training loss can be formulated as:
\begin{align}\label{b_loss}
\mathcal{L}_0[b_{\theta_0}, c_{\theta_0}] =& \int_{0}^{1}\mathbb{E}_{\nu_{t}} \big| b_{\theta_0}(x,t)\cdot \nabla U(x,t) +\partial_t U(x,t)  - \nabla \cdot b_{\theta_0}(x,t) +  c_{\theta_0}(t) \big|^2 \mathrm{d}t.
\end{align}
Direct sampling from $\nu_{t}$ is generally intractable, since it is precisely the target of the Boltzmann sampler. 

NETS \citep{albergo2024nets} uses the following time-continuous variant of annealed importance sampling based on the Jarzynski equality, to estimate expectations under $\nu_{t}$. For a given velocity $\hat{b}(x,t)\in \mathbb{R}^n$, let $(X_t^{\hat{b}}, A_t^{\hat{b}})$ solve the coupled system:
\begin{align}
&\mathrm{d} X_t^{\hat{b}} = -\epsilon_t \nabla U(X_t^{\hat{b}},t)\mathrm{d}t + \hat{b}(X_t^{\hat{b}},t)\mathrm{d}t + \sqrt{2\epsilon_t} \mathrm{d} W_t,  \label{non_equil_X}\\
&\mathrm{d} A_t^{\hat{b}} = \nabla\cdot \hat{b}(X_t^{\hat{b}},t)\mathrm{d}t - \nabla U(X_t^{\hat{b}},t)\cdot \hat{b}(X_t^{\hat{b}},t)\mathrm{d}t - \partial_t U(X_t^{\hat{b}},t)\mathrm{d}t, \label{non_equil_A}
\end{align}
with initial conditions $X_0^{\hat{b}} \sim \nu_0$, $A_0^{\hat{b}} = 0$. Here $\epsilon_t$ is a time-dependent diffusion coefficient, and $W_t$ is the standard Wiener process. Then, for any $t\in[0,1]$ and test function $h:\mathbb{R}^n \to \mathbb{R}$, we have
\begin{equation}\label{Jar_reweight}
\mathbb{E}_{\nu_{t}}h(x)=\int h(x)p(x,t)\mathrm{d}x = \frac{\mathbb{E}[e^{A_t^{\hat{b}}}h(X_t^{\hat{b}})]}{\mathbb{E}[e^{A_t^{\hat{b}}}]},
\end{equation}
where the expectations in the last term are taken with respect to the law of $(X_t^{\hat{b}}, A_t^{\hat{b}})$. It also follows that
\begin{equation}\label{Jar_eq}
\log Z(t) -\log Z(0)=\log \mathbb{E}[e^{A_t^{\hat{b}}}],
\end{equation}
which is referred to as the Jarzynski equality. 
See \citet[Proposition~2.4]{albergo2024nets} for details, and refer to \citet{PhysRevLettVaikuntanathan,pmlr-v235-tian24c,vargas2024transport} for related work.

Equation \eqref{Jar_reweight} gives an exact reweighting identity for expectations under $\nu_t$, regardless of the choice of $\hat{b}$. In practice, the expectations on the right-hand side are approximated using Monte Carlo trajectories, yielding a self-normalized importance sampling estimator. The variance of this estimator depends on the mismatch between the target density $p(\cdot,t)$ and the marginal law induced by $X_t^{\hat{b}}$; consequently, a poorly chosen (e.g., random) $\hat{b}$ can produce highly variable, potentially degenerate weights. If $\hat{b}$ satisfies the inhomogeneous transport equation \eqref{PINN_full_point}, the path weights $\exp(A_t^{\hat{b}})$ become constant across trajectories, resulting in a zero-variance reweighting. 
NETS proposes parameterizing $\hat{b}=b_\theta$ and optimizing $\theta$ so that the variance of the resulting Monte Carlo estimators is progressively reduced during training. 
This requires simulating trajectories of the coupled non-equilibrium dynamics in \eqref{non_equil_X}--\eqref{non_equil_A}; see lines 3--8 of Algorithm~1 in \citet{albergo2024nets} for the simulation procedure. In particular, when the drift is chosen as the zero field, $\hat{b}\equiv 0$, the resulting weight degeneracy can lead to a large estimator variance. To reduce the variance in such cases, we use an optional resampling variant in which particles are periodically resampled according to their normalized Jarzynski weights; see \citep[Algorithm~1]{tan2025scalable}. This variant is used in the alanine dipeptide experiment in Section~\ref{Sec:alanine}.

\section{Methods}

To construct a generator for the target density $\rho(y)$, we develop a reduced dynamical system defined on the CV space. Specifically, we map the mathematical objects from Section \ref{Background} to their counterparts in the reduced model, and derive the partial differential equation (PDE) satisfied by the velocity of this reduced dynamics.

\subsection{Reduced dynamics from the transport map in the configuration space}

For $X_t \sim p(x,t)$, the probability density of the random variable $Y_t=\xi(X_t)$ can be written as
\begin{equation}
\rho(y,t) = \int_{\Sigma_{y}} p(x,t) |J_{\xi}(x)|^{-1} \mathrm{d} \sigma_{\Sigma_{y}}(x).
\end{equation}
We denote the corresponding distribution by $\mathrm{d} \mu_{t}(y)  =\rho(y,t) \mathrm{d}y$. Following the idea in flow matching, we aim to find a velocity field $u(y,t)$ such that the marginal distribution of the ODE solution:
\begin{equation}\label{ode_red}
    \frac{\mathrm{d}}{\mathrm{d}t} Y_t= u(Y_t, t)
\end{equation}
satisfies $Y_t\sim \rho(y,t)$ for all $t \in [0, 1]$. In particular, we have $Y_1 \sim \rho(y)$. Thus, once $u(y,t)$ is obtained, sampling from the target distribution can be achieved by solving the ODE \eqref{ode_red} starting with $Y_0=\xi(X_0)$, where $X_0\sim p(x,0)$ is drawn from the prior distribution.

The velocity field $u(y,t)$ and the density $\rho(y,t)$ are related through the Liouville equation, i.e. $\partial_t \rho(y,t) + \nabla \cdot (\rho(y,t) u(y,t)) = 0$. Let $F(y,t) = -\log \rho(y,t)$ represent the free energy at time $t$. Similar to \eqref{PINN_full_point}, the Liouville equation can be rewritten as
\begin{equation}\label{PINN_red_point}
\partial_t F(y,t) + u(y,t) \cdot \nabla F(y,t) - \nabla \cdot u(y,t) =0.
\end{equation}

\begin{figure}[htbp]
\centering
\begin{subfigure}[b]{0.3\textwidth}
\centering
\includegraphics[width=0.98\linewidth]{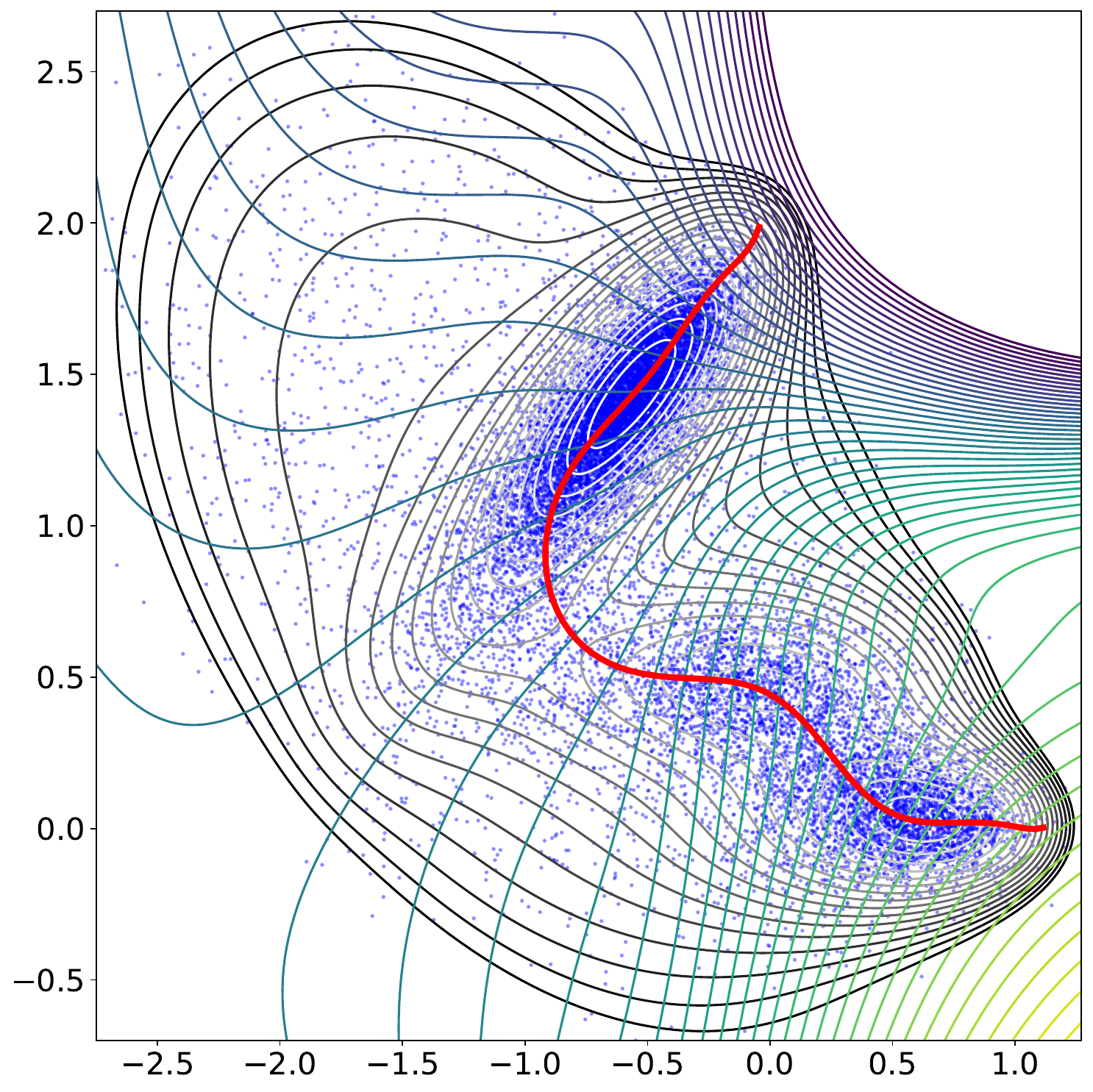}
\caption{Landscape and learned CV}\label{fig:MB_CV}
\end{subfigure}
\begin{subfigure}[b]{0.3\textwidth}
\centering
\includegraphics[width=0.98\linewidth]{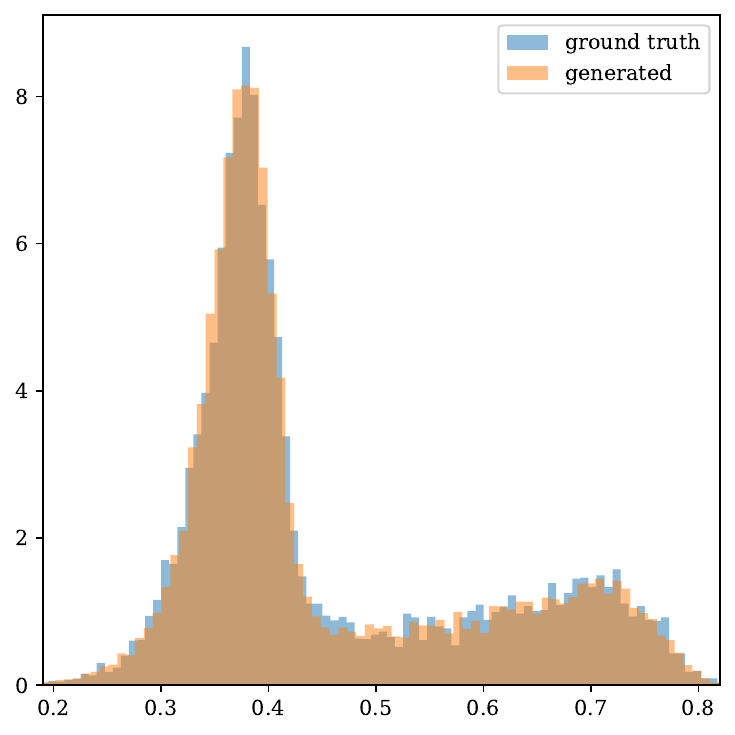}
\caption{NETS-P}\label{fig:MB_full}
\end{subfigure}
\begin{subfigure}[b]{0.3\textwidth}
\centering
\includegraphics[width=0.98\linewidth]{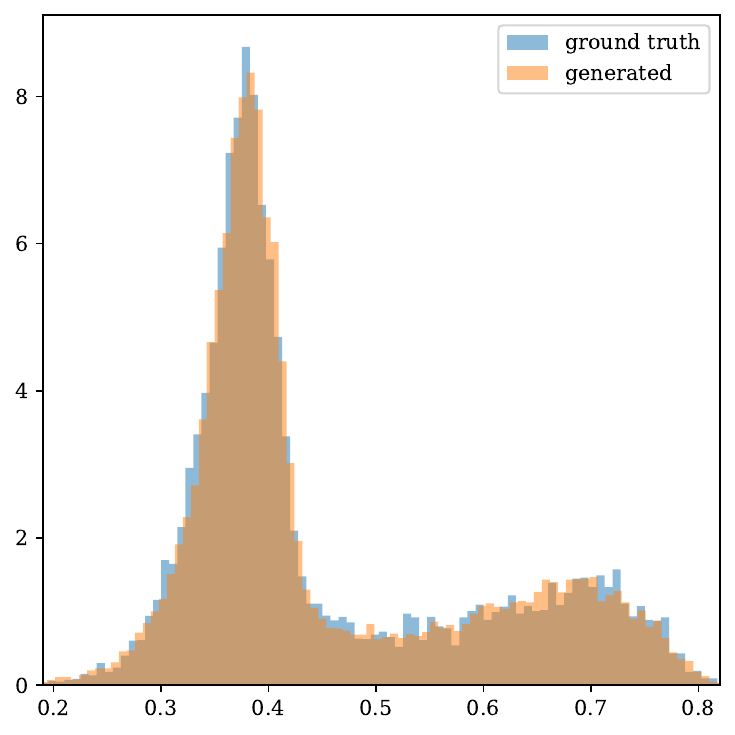}
\caption{FES-FM}\label{fig:MB_red}
\end{subfigure}
\caption{Results for Müller-Brown potential: (a) Illustration of the Müller-Brown potential, where the red curve denotes the transition path calculated via the string method. The blue dots represent samples from the Langevin dynamics, and we regard these samples as the ground truth distribution. The black contour lines correspond to the contours of the Müller-Brown potential, while the contour lines colored by the \textit{viridis} colormap indicate the contours of the collective variable. (b) Empirical densities of CV samples obtained by projecting NETS samples (blue) and the ground truth (red). (c) Empirical densities of the CV from our FES-FM (blue) and the ground truth (red).}
\end{figure}

\subsection{Inhomogeneous transport equation for the reduced dynamics}

The derivatives $\partial_t F(y,t)$ and $\nabla F(y,t)$ in \eqref{PINN_red_point} can be expressed as expectations with respect to the time-dependent measure $ \mu_{\Sigma_{y},t}$ on $\Sigma_{y}$:
\begin{equation}\label{manifold_dense_t}
\mathrm{d} \mu_{\Sigma_{y},t}(x) = \frac{1}{\rho(y,t)} \cdot p(x,t) |J_{\xi}(x)|^{-1} \mathrm{d} \sigma_{\Sigma_{y}}(x). 
\end{equation}
In particular, $\nabla F(y,t)$ can be computed as 
\begin{equation}\label{mean_force_identity}
\nabla F(y,t) = \mathbb{E}_{\mu_{\Sigma_{y},t}} D(x,t),
\end{equation}
where $D(x,t)$, known as the local mean force \citep{stoltz2010free}, is the difference of the force exerted on the system along the CVs and a term related to the curvature of the manifold $\xi(x)=y$ \citep{ciccotti2008projection,stoltz2010free}. For $1\leq i \leq d$, its $i$-th component is given by
\begin{equation}
\begin{aligned}
& D_{i}(x,t) = \sum_{j=1}^{n}\nabla \xi(x)^{\dagger}_{i,j} \partial_{x_j} U(x,t) -  \sum_{j=1}^{n}\partial_{x_{j}}\nabla \xi(x)^{\dagger}_{i,j},
\end{aligned}
\end{equation}
where $\nabla \xi(x)^{\dagger}=(\nabla \xi(x)^{T} \nabla \xi(x))^{-1}\nabla \xi(x)^T\in \mathbb{R}^{d\times n}$ represents the Moore-Penrose inverse \citep{Penrose_1955} of $\nabla \xi(x)$. Moreover, we can verify the following expression for $\partial_t  F(y,t)$:
\begin{equation}\label{partial_t_F}
\partial_t F(y,t) = \mathbb{E}_{\mu_{\Sigma_{y},t}} [\partial_{t} U(x,t)] + \partial_{t} \log Z(t).
\end{equation}
Therefore, the inhomogeneous transport equation \eqref{PINN_red_point} becomes 
\begin{align}
& \mathbb{E}_{\mu_{\Sigma_{y},t}} \partial_{t} U(x,t) +u(y,t) \cdot \mathbb{E}_{\mu_{\Sigma_{y},t}} D(x,t)  - \nabla \cdot u(y,t) + \partial_{t} \log Z(t)  =0.
\end{align}

We parametrize the time-dependent velocity field $u$ by a neural network $u_{\theta_1}$. Since $\partial_{t} \log Z(t)$ is an intractable function, we also use a neural network $c_{\theta_1}$ to approximate it. For any $y$ and $t$, we aim to find the vector field $u_{\theta_1}$ and $c_{\theta_1}$ such that
\begin{align}\label{PINN_red_ori}
&\mathbb{E}_{\mu_{\Sigma_{y},t}} \partial_{t} U(x,t) +u_{\theta_1}(y,t) \cdot \mathbb{E}_{\mu_{\Sigma_{y},t}} D(x,t) - \nabla \cdot u_{\theta_1}(y,t) + c_{\theta_1}(t)=0.
\end{align}

However, the computation of expectations with respect to $\mu_{\Sigma_{y},t}$ in \eqref{PINN_red_ori} is challenging, due to the difficulty in obtaining samples from the distribution supported on the manifold $\Sigma_{y}$. Therefore, we cannot directly define the PINN loss in the same way as \eqref{b_loss}. Appendix~\ref{compare_summary} provides a complete comparison list for the full NETS model and the reduced FES-FM model.

\subsection{Training loss}

To find $u_{\theta_1}$ and $c_{\theta_1}$ satisfying \eqref{PINN_red_ori}, we leverage the properties of expectations to manipulate the integral term with respect to $\mu_{\Sigma_{y},t}$. We first introduce an auxiliary neural network $v_{\theta_1}(y,t)$ to approximate $\mathbb{E}_{\mu_{\Sigma_{y},t}} D(x,t)$. For any $y,t$, we define $\ell_1$ as follows:
\begin{align}
& \ell_1[v_{\theta_1}](y,t) = \mathbb{E}_{\mu_{\Sigma_{y},t}}  \big\| D(x,t)-v_{\theta_1}(y,t) \big\|^{2} .
\end{align}
Note that the expectation $\mathbb{E}_{\mu_{\Sigma_{y},t}}$ is taken with respect to $x$ for fixed $y$ and $t$. By the property of squared-loss minimization, the pointwise minimizer of $\ell_1[v_{\theta_1}](y,t)$ is the mean of $D(x,t)$ under $\mu_{\Sigma_{y},t}$, i.e.,
\begin{equation}\label{minimizer_l1}
v_{\theta_1}(y,t)= \mathbb{E}_{\mu_{\Sigma_{y},t}} D(x,t).
\end{equation}

Subsequently, we define
\begin{align}
&\ell_2[u_{\theta_1},c_{\theta_1},v_{\theta_1}](y,t) =\mathbb{E}_{\mu_{\Sigma_{y},t}}\big| \partial_{t} U(x,t) +u_{\theta_1}(y,t) \cdot v_{\theta_1}(y,t) - \nabla \cdot u_{\theta_1}(y,t) + c_{\theta_1}(t) \big|^{2},
\end{align}
and its pointwise minimizer with respect to $u_{\theta_1}(y,t)$, $c_{\theta_1}(t)$, and $v_{\theta_1}(y,t)$ satisfies
\begin{align}
&\mathbb{E}_{\mu_{\Sigma_{y},t}}\partial_{t} U(x,t)  = - u_{\theta_1}(y,t) \cdot v_{\theta_1}(y,t)  + \nabla \cdot u_{\theta_1}(y,t)- c_{\theta_1}(t). \label{minimizer_l2}
\end{align}

Note that the intersection of the solution sets of \eqref{minimizer_l1} and \eqref{minimizer_l2} is non-empty, and this intersection is exactly the solution set of the transport equation \eqref{PINN_red_ori}. Therefore, the optimal $u_{\theta_1}(y,t),c_{\theta_1}(t),v_{\theta_1}(y,t)$ that minimize $\ell_1+\lambda\ell_2$ ($\lambda > 0$) satisfies both \eqref{minimizer_l1} and \eqref{minimizer_l2} simultaneously, and hence satisfies \eqref{PINN_red_ori}. Based on this fact, we define the following loss function
\begin{align}\label{loss_u_ori}
&\mathcal{L}_1[u_{\theta_1},c_{\theta_1},v_{\theta_1}] =  \int_{0}^{1}\mathbb{E}_{\mu_{t}} \big[ \ell_1[v_{\theta_1}](y,t) \big] \mathrm{d}t  + \lambda \int_{0}^{1}\mathbb{E}_{\mu_{t}} \big[ \ell_2[u_{\theta_1},c_{\theta_1},v_{\theta_1}](y,t) \big] \mathrm{d}t. 
\end{align}

Note that $\mu_{\Sigma_{y},t}$ can be regarded as the conditional distribution of $\nu_{t}$ given $\xi(x)=y$, such that $\mathbb{E}_{\mu_{t}}\mathbb{E}_{\mu_{\Sigma_{y},t}} = \mathbb{E}_{\nu_{t}}$ holds. We can thus rewrite the loss function $\mathcal{L}_1$ as
\begin{align}\label{loss_u}
\mathcal{L}_1[u_{\theta_1},c_{\theta_1},v_{\theta_1}] =& \int_{0}^{1}\mathbb{E}_{\nu_{t}} \big[ \mathcal{R}[u_{\theta_1},c_{\theta_1},v_{\theta_1}](x,t)\big] \mathrm{d}t,
\end{align}
where
\begin{align}
\mathcal{R}[u_{\theta_1},c_{\theta_1},v_{\theta_1}](x,t) = &
\lambda \big| \partial_{t} U(x,t) +u_{\theta_1}(\xi(x),t) \cdot v_{\theta_1}(\xi(x),t) - \nabla \cdot u_{\theta_1}(\xi(x),t) + c_{\theta_1}(t) \big|^{2} \notag \\
&+ \big\| D(x,t)-v_{\theta_1}(\xi(x),t) \big\|^{2} ,
\end{align}
and the expectation $\mathbb{E}_{\nu_{t}}[\cdot]$ can be estimated via \eqref{Jar_reweight}.

\subsection{Training pipeline}

Algorithm~\ref{algo_train} summarizes the main training process, where a fixed drift $\hat{b}$ is used in the non-equilibrium dynamics; if a warm-up process is used, $\hat{b}$ is obtained by Algorithm~\ref{algo_warmup} in Appendix~\ref{sec:warmup_process}, and otherwise $\hat{b}$ is specified directly.

Similar to NETS, simulating the trajectories of $X_t,A_t$ using \eqref{non_equil_X} and \eqref{non_equil_A} is required, and $X_t,A_t$ are detached from the computational graph when taking a gradient step in $\mathcal{L}_1$. The loss \eqref{loss_u} is optimized by leveraging \eqref{Jar_reweight} to compute the reweighting factor. Algorithm~\ref{algo_gen} summarizes the sampling procedure.

\begin{algorithm}[ht]
\caption{FES-FM: Training process}\label{algo_train}
\begin{algorithmic}[1]
\STATE \textbf{Initialize:} neural networks $u_{\theta_1}$, $c_{\theta_1}$, $v_{\theta_1}$; a fixed drift $\hat{b}$; training epochs $N_{\mathrm{epoch}}$; time steps $K$
\FOR{epoch=$1, \dots, N_{\mathrm{epoch}}$}
    \STATE Generate trajectories $\{(X_{t_k},A_{t_k})\}_{1\leq k\leq K}$ by solving \eqref{non_equil_X}--\eqref{non_equil_A} with $\hat{b}$
    \STATE Calculate $\mathcal{L}_1[u_{\theta_1}, c_{\theta_1}, v_{\theta_1}]$ in \eqref{loss_u}, where the expectation $\mathbb{E}_{\nu_{t}}[\cdot]$ is estimated using \eqref{Jar_reweight} with $\{(X_{t_k},A_{t_k})\}$
    \STATE Update the parameters of $u_{\theta_1}, c_{\theta_1}, v_{\theta_1}$ by gradient descent on $\mathcal{L}_1$
\ENDFOR
\STATE \textbf{Return:} the velocity field $u_{\theta_1}$ of the reduced model
\end{algorithmic}
\end{algorithm}

\begin{algorithm}[ht]
\caption{FES-FM: Sampling process}\label{algo_gen}
\begin{algorithmic}[1]
\REQUIRE trained velocity field $u_{\theta_1}$, the number of steps $K_{0}$
\STATE $\Delta t=1/K_{0}$
\STATE Sample $X_0 \sim p(x,0)$, and set $Y_0=\xi(X_0)$
\FOR{$i=0, \dots, K_{0}-1$}
\STATE $Y_{(i+1)\Delta t} = Y_{i\Delta t} + u_{\theta_1}(Y_{i\Delta t}, i\Delta t)\Delta t $
\ENDFOR
\STATE \textbf{Return:} CV samples $Y_1$
\end{algorithmic}
\end{algorithm}

\subsection{Computation of the free energy surface}

Beyond generating CV samples, FES-FM can also be used to estimate the free energy surface itself. Recall from \eqref{mean_force_identity} that the local mean force satisfies
$\mathbb{E}_{\mu_{\Sigma_{y},t}}D(x,t)=\nabla F(y,t)$.
This identity suggests learning a scalar-valued neural network $F_{\theta}(y,t)$ whose spatial gradient matches the conditional average of $D(x,t)$. For fixed $y$ and $t$, the minimizer of
\begin{equation}\label{loss_ell_fes_cal}
\mathbb{E}_{\mu_{\Sigma_{y},t}} \big\|D(x,t)-\nabla F_{\theta}(y,t)\big\|^{2}
\end{equation}
satisfies
\begin{equation}
\nabla F_{\theta}(y,t)=\mathbb{E}_{\mu_{\Sigma_{y},t}}D(x,t).
\end{equation}
Therefore, at $t=1$, $F_{\theta}(y,1)$ estimates the free energy surface associated with the target CV distribution, up to an additive constant.

As in the derivation of \eqref{loss_u}, the conditional expectation in \eqref{loss_ell_fes_cal} is difficult to compute directly. We instead average over $\mu_t$ and use the identity $\mathbb{E}_{\mu_t}\mathbb{E}_{\mu_{\Sigma_{y},t}}=\mathbb{E}_{\nu_t}$ to obtain the practical training objective
\begin{equation}\label{loss_fes_cal}
\mathcal{L}_{\mathrm{FES}}[F_{\theta}]
=\int_{0}^{1}\mathbb{E}_{\nu_t}\big\|D(x,t)-\nabla F_{\theta}(\xi(x),t)\big\|^{2}\mathrm{d}t.
\end{equation}
The expectation in \eqref{loss_fes_cal} can be estimated with the same Jarzynski reweighting identity \eqref{Jar_reweight}. Consequently, once the drift $\hat{b}$ for the non-equilibrium dynamics is fixed, $F_{\theta}$ can be trained from the same non-equilibrium trajectories used in Algorithm~\ref{algo_train}, either as an auxiliary objective during the main training process or as a post-processing step. After training, the free energy difference between two configurations $x^{(1)}$ and $x^{(2)}$ can be estimated by
\begin{equation}
F_{\theta}(\xi(x^{(1)}),1)-F_{\theta}(\xi(x^{(2)}),1).
\end{equation}
Numerical reconstructions of the free energy surface based on this objective are reported in Section~\ref{sec:fes_cal_results} and Figure~\ref{fig:FES_cal}.

\section{Hessian-informed harmonic prior distribution}\label{Sec:hessian}

In this section, we introduce a prior distribution for many-particle systems, termed the \textbf{Hessian-informed harmonic prior distribution}. Our goal is to design a prior that is invariant under global rotations and translations (i.e., $\mathrm{E}(3)$-invariant) while producing configurations that are physically meaningful and numerically stable for downstream dynamics.

Consider a system of $M$ atoms in $\mathbb{R}^3$. Its configuration is denoted by $x=(x_1^T,\dots,x_M^T)^T$, which lies in $\mathbb{R}^{n}$, where $n=3M$. 
To ensure that our model is $\mathrm{E}(3)$-invariant, we choose a prior that is invariant under global rotations and translations.
A widely used option in generative modeling is the mean-free Gaussian prior~\citep{kohler2020equivariant}, which enforces the desired symmetry. However, symmetry alone is insufficient for the sampling task considered in this paper, as it may yield configurations that are not physically meaningful. This issue is particularly important here because our method involves an interpolant between $U_0$ and $U_{\mathrm{target}}$, and thus explores intermediate distributions. To better reflect physically plausible configurations and improve numerical stability, we propose a prior that encodes curvature information through the Hessian.

Since a nontrivial normalizable translation-invariant probability distribution on $\mathbb{R}^{n}$ does not exist~\citep{pmlr-v202-yim23a,xu2026quotientspace}, we eliminate the global translation degree of freedom by restricting the prior to the center-of-mass subspace $\mathcal{P}$, defined as $\mathcal{P}=\{(x_1^T,\dots,x_M^T)^T\in \mathbb{R}^{n}\mid \sum_{i=1}^{M}x_i=0\}$. We then construct an $\mathrm{O}(3)$-invariant distribution on $\mathcal{P}$, which yields an $\mathrm{E}(3)$-invariant prior on the original configuration space. We next present the formal definition of the Hessian-informed harmonic prior distribution and discuss its properties.

Let $x_0 \in \mathbb{R}^n$ denote a local minimum of $U_{\mathrm{target}}(x)$, which satisfies $\nabla U_{\mathrm{target}}(x_0) = 0$ and $x_0\in \mathcal{P}$, and the Hessian matrix of $U_{\mathrm{target}}(x)$ at $x_0$ is denoted as $H(x_0) \in \mathbb{R}^{n \times n}$. Due to the rotational and translational invariance of $U_{\mathrm{target}}(x)$, the matrix $H(x_0)$ is singular. In Theorem~\ref{Thm_null_space_hessian}, we prove that $\operatorname{rank}(H(x_0)) = n - 6$ and discuss its null space in detail. We decompose $H(x_0)$ as $H(x_0) = P S P^T$, where $P \in \mathbb{R}^{n \times (n-6)}$ satisfies $P^T P = I_{n-6}$ and $S = \operatorname{diag}(s_1, s_2, \dots, s_{n-6})$.

The sampling procedure of the prior distribution is described as follows. First, we sample $R$ from the (Haar) uniform measure on $\mathrm{O}(3)$. Subsequently, we sample a random variable $\epsilon \sim \mathcal{N}(0, I_{n-6})$ and compute $x$ according to $x = (I_{M}\otimes R)\left(x_0 + P S^{-\frac{1}{2}} \epsilon\right)$. Here, $\otimes$ denotes the Kronecker product. 
Equivalently, conditional on $R$, the prior is Gaussian with mean $(I_M\otimes R)x_0$ and covariance $(I_M\otimes R) P S^{-1} P^T (I_M\otimes R)^T$; marginalizing over random rotations yields an $\mathrm{O}(3)$-invariant distribution.

Conditioned on a rotation, the distribution is approximately Gaussian on the $(n-6)$-dimensional subspace spanned by the nontrivial Hessian modes; in practice, we approximately evaluate the corresponding quadratic potential after aligning $x$ to $x_0$. Specifically, the potential of the Hessian-informed harmonic prior distribution is estimated as:
\begin{align}
& U_0(x) =  \frac{1}{2} ( \tilde{x}-  x_0)^T H (x_0) ( \tilde{x}- x_0 ) , \label{prior_potential}
\end{align}
where $\tilde{x}=(I_{M}\otimes R_0^{*}(x))x$, and $R_0^{*}(x) = \arg\min_{R\in \mathrm{O}(3)} \| (I_{M}\otimes R) x - x_0 \|^2$. The optimal $R_0^{*}(x)$ can be computed using the Kabsch algorithm \citep{Kabsch1976ASF}. In Theorem \ref{thm_invariant_prior_pot}, we prove that $U_0(x)$ is $\mathrm{O}(3)$-invariant.

In some applications, generating samples from both chiral states can be undesirable. Nevertheless, samples with the unwanted chirality can be readily detected and easily corrected \citep{klein2023equivariant}. Alternatively, one can also employ an $\mathrm{SE}(3)$-invariant prior distribution.

\section{Experiments}

In this section, we evaluate our method on the Müller-Brown potential and in high-dimensional scenarios. We also carry out experiments on many-particle systems to explore potential applications in molecular contexts, and employ the Hessian-informed harmonic distribution introduced above as the prior. In addition, we introduce an alanine dipeptide experiment in implicit solvent as a molecular benchmark. For the synthetic benchmark experiments, since NETS can generate samples following the Boltzmann distribution, we adopt the points produced by NETS and projected onto the CV space as the baseline (denoted as NETS-P). The alanine dipeptide experiment is instead compared against molecular-dynamics reference distributions. Figure \ref{fig:sampling_demo} illustrates the sampling procedures of FES-FM and NETS-P. In Section \ref{sec:abl}, we conduct ablation studies on $N_{\mathrm{pre}}$ and $\epsilon$.

\begin{figure}[ht]
\centering
\includegraphics[width=0.8\linewidth]{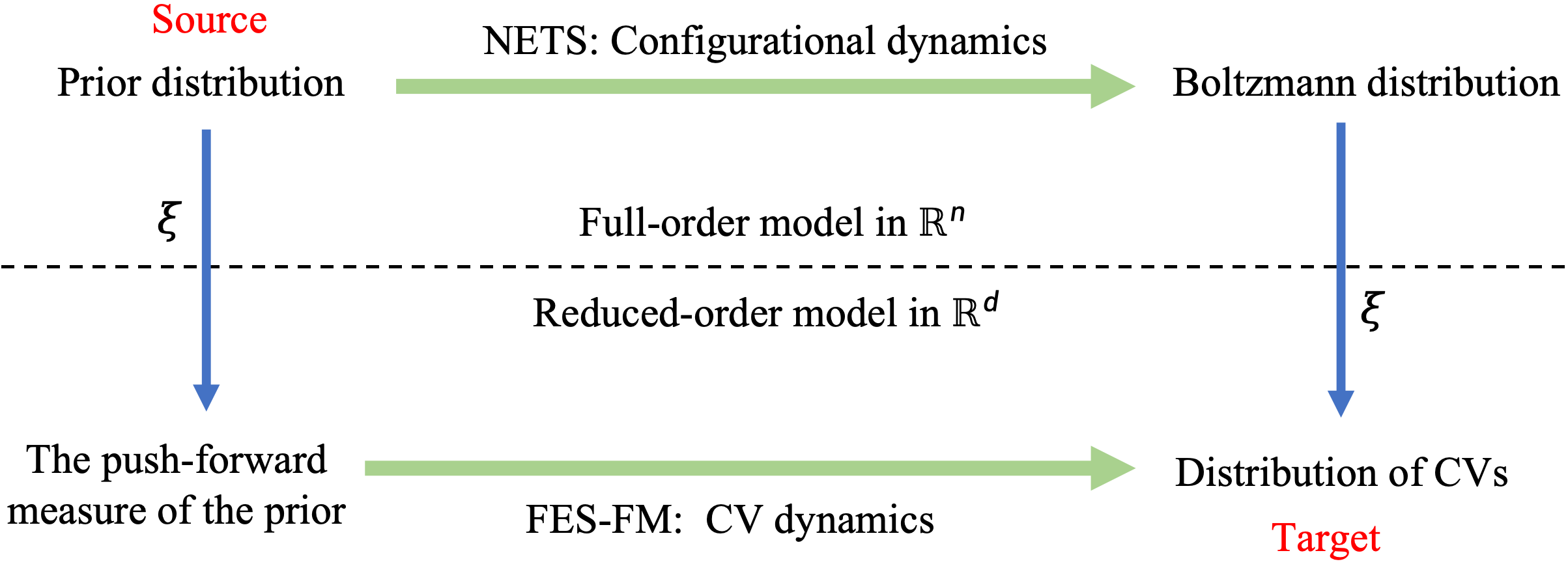}
\caption{Comparison of sampling workflows between FES-FM and NETS-P. Starting from the prior distribution in the conformational space, NETS-P first evolves the high-dimensional dynamics before projecting them onto the CV space. In contrast, FES-FM first projects the prior distribution onto the CV space and then evolves the low-dimensional dynamics.}\label{fig:sampling_demo}
\end{figure}

\subsection{Müller-Brown potential}
The Müller-Brown surface is a well-established benchmark potential energy landscape, given by
\begin{align}
&U_{\mathrm{target}}(x_1, x_2) = \sum_{i=1}^{4} A_i \exp \big(  a_i (x_1 - x_{1,i}^0)^2  +b_i (x_1 - x_{1,i}^0)(x_2 - x_{2,i}^0)+ c_i (x_2 - x_{2,i}^0)^2 \big).
\end{align}
The target distribution and the corresponding potential contours are shown in Figure~\ref{fig:MB_CV}. 
For this system, we learn a CV map using samples on the transition path. 
The CV map is represented by a neural network $y=\xi_{\theta}(x)$. 
Figure~\ref{fig:MB_CV} also displays the transition path and the contours of the learned CV.
Figure \ref{fig:MB_red} shows that the generated CV distribution agrees well with the reference CV distribution, even for a highly nonlinear CV. Numerical results are shown in Table~\ref{table:MB_res}.

\begin{table}[htbp]
\centering
\caption{Results for the Müller-Brown potential. Our method achieves comparable raw accuracy while maintaining faster sampling.}\label{table:MB_res}
\begin{tabular}{lccc}
\toprule
 Method & Time $\downarrow$ & Error $\downarrow$ & ACC/Time $\uparrow$ \\ 
\midrule
NETS-P & 4.60e-1{\scriptsize $\pm$2.35e-2} & \cellcolor{gray!15} 4.00e-3{\scriptsize $\pm$1.23e-3} & 6.09e+2{\scriptsize $\pm$2.20e+2} \\ 
FES-FM & \cellcolor{gray!15} 3.73e-1{\scriptsize $\pm$1.24e-2} & 4.11e-3{\scriptsize $\pm$1.65e-3} & \cellcolor{gray!15} 7.73e+2{\scriptsize $\pm$3.08e+2} \\
\bottomrule
\end{tabular}
\end{table}

\subsection{High-dimensional example}\label{Sec:DW}

To validate the scalability of our method, we consider a potential function defined in a high-dimensional space with dimensions $n=50, 100, 200$. This distribution exhibits a double-well structure along one specific direction, while following a Gaussian distribution along all other directions. Specifically, we define the potential function as
\begin{equation}
U_{\mathrm{target}}(x)=U_{*}(A_1^T x)+\frac{1}{2}\sum_{i=2}^{n}\|A_i^T x\|^2,
\end{equation}
where $A=(A_1, A_2,\dots,A_n)^{T} \in \mathbb{R}^{n \times n}$ is an orthogonal matrix, and $U_{*}(x) = \frac{1}{5}x^4 -\frac{6}{5}x^2 - \frac{1}{10}x$ is a double-well potential. The CV map is chosen as $\xi(x)=A_1^T x$, which reflects the characteristics of the non-Gaussian eigen-direction of the system.

As presented in Table \ref{table:DW_res}, our method significantly improves sampling speed while maintaining accuracy, as it only requires solving a low-dimensional ODE \eqref{ode_red}, whereas NETS must solve an ODE in the full $n$-dimensional space.

\begin{table}[h]
    \centering
    \caption{Results for the high-dimensional example. \textbf{Time} denotes the time (in seconds) required for sample generation (i.e., solving the ODE); \textbf{Error} represents the 1-Wasserstein distance between the generated samples and the ground truth; and \textbf{ACC/Time} stands for accuracy per unit time, which is calculated as $(\text{Time} \cdot \text{Error})^{-1}$. The results are reported as the mean and standard deviation over five independent runs. The shaded cells indicate that our method achieves the best performance.}\label{table:DW_res}
    \begin{tabular}{llcccc}
    \toprule
    & Method & Time $\downarrow$ & Error $\downarrow$ & ACC/Time $\uparrow$ \\ 
    \midrule
    \multirow{2}{*}{$n=50$}  & NETS-P & 1.89e+0{\scriptsize $\pm$9.05e-3} & 2.00e-2{\scriptsize $\pm$7.99e-3} & 3.01e+1{\scriptsize $\pm$9.77e+0} \\ 
    &  FES-FM &  \cellcolor{gray!15} 2.93e-1{\scriptsize $\pm$1.21e-2} &  \cellcolor{gray!15} 1.62e-2{\scriptsize $\pm$4.32e-3} &  \cellcolor{gray!15} 2.28e+2{\scriptsize $\pm$6.57e+1} \\ 
    \midrule
    \multirow{2}{*}{$n=100$} & NETS-P & 2.00e+0{\scriptsize $\pm$1.43e-2} & 2.37e-2{\scriptsize $\pm$9.90e-3} & 2.46e+1{\scriptsize $\pm$9.02e+0} \\ 
    &  FES-FM &  \cellcolor{gray!15} 2.90e-1{\scriptsize $\pm$2.39e-2} &  \cellcolor{gray!15} 1.91e-2{\scriptsize $\pm$3.86e-3} &  \cellcolor{gray!15} 1.90e+2{\scriptsize $\pm$4.55e+1} \\ 
    \midrule
    \multirow{2}{*}{$n=200$}  & NETS-P  & 2.28e+0{\scriptsize $\pm$6.74e-3} & 2.48e-2{\scriptsize $\pm$1.44e-2} & 2.69e+1{\scriptsize $\pm$1.81e+1} \\ 
    &   FES-FM  &  \cellcolor{gray!15} 2.80e-1{\scriptsize $\pm$1.78e-2} &  \cellcolor{gray!15} 1.60e-2{\scriptsize $\pm$6.46e-3} &  \cellcolor{gray!15} 2.73e+2{\scriptsize $\pm$1.35e+2} \\ 
    \bottomrule
    \end{tabular}
    \end{table}

\subsection{Many-particle systems}

The potential for the many-particle systems is constructed in analogy to empirical potential energy functions in molecular dynamics. We evaluate our method on two synthetic systems: a three-particle system in $\mathbb{R}^2$ (denoted $\mathbb{R}^2$-3P) and a four-particle system in $\mathbb{R}^3$ (denoted $\mathbb{R}^3$-4P). We find that using a mean-free Gaussian prior leads to numerical instability and unsatisfactory performance in our experiments. This is mainly because its samples often correspond to non-physical configurations, which in turn destabilize training and generation. Therefore, we adopt the Hessian-informed harmonic prior introduced in Section~\ref{Sec:hessian}. Although Section~\ref{Sec:hessian} focuses on the $\mathbb{R}^3$ case, the same approach easily extends to constructing the corresponding prior for $\mathbb{R}^2$.

\begin{figure}[h]
\centering
\begin{subfigure}[b]{0.49\textwidth}
\centering
\includegraphics[width=0.9\linewidth]{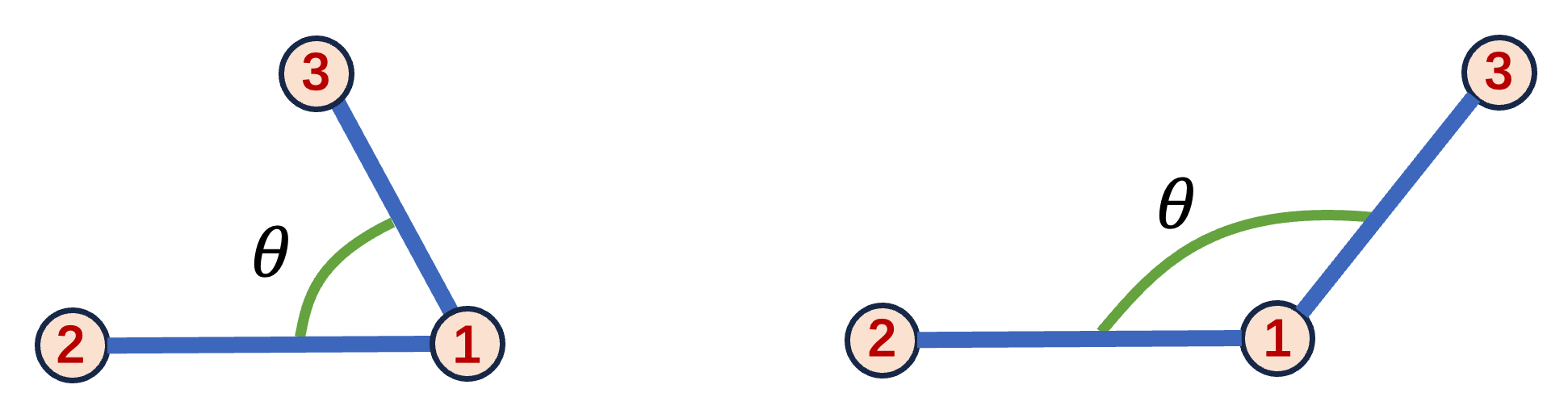}
\caption{$\mathbb{R}^2$-3P}\label{fig:R2_3P_demo}
\end{subfigure}
\begin{subfigure}[b]{0.49\textwidth}
\centering
\includegraphics[width=0.9\linewidth]{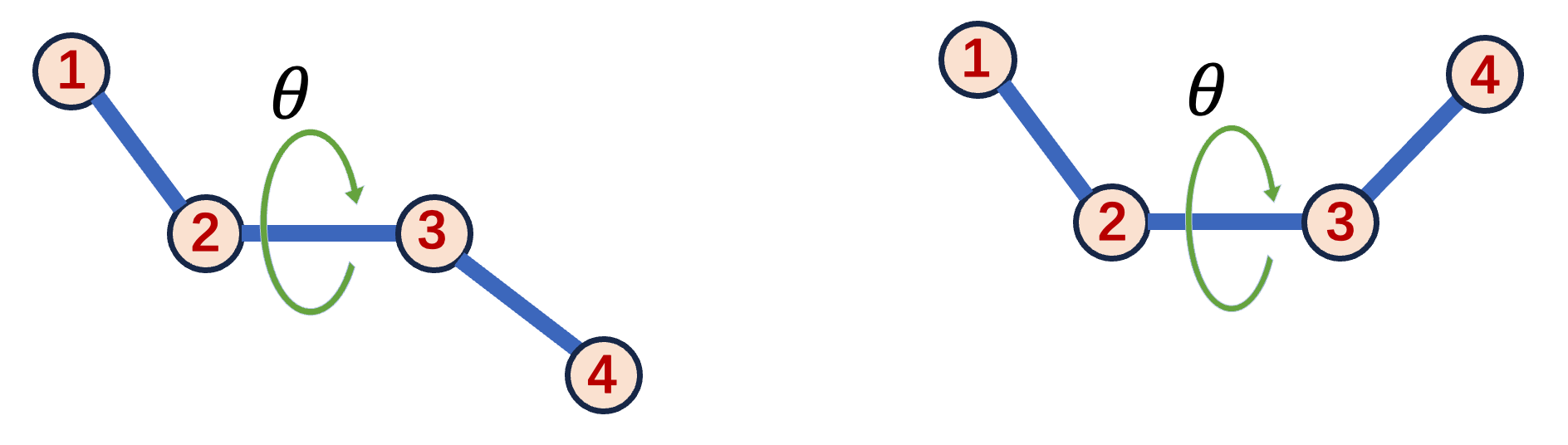}
\caption{$\mathbb{R}^3$-4P}\label{fig:R3_4P_demo}
\end{subfigure}
\caption{Illustration of multi-particle systems. (a) A three-particle system in a 2D plane, where the CV is the cosine of the angle centered at particle 1. (b) A four-particle system in 3D space, where the CV is the cosine of the dihedral angle (rotation angle) defined by the axis of particles 2 and 3.}\label{fig:many_particle_demo}
\end{figure}

\begin{figure}[htbp]
    \centering
    \begin{subfigure}[b]{0.23\textwidth}
    \centering
    \includegraphics[width=0.9\linewidth]{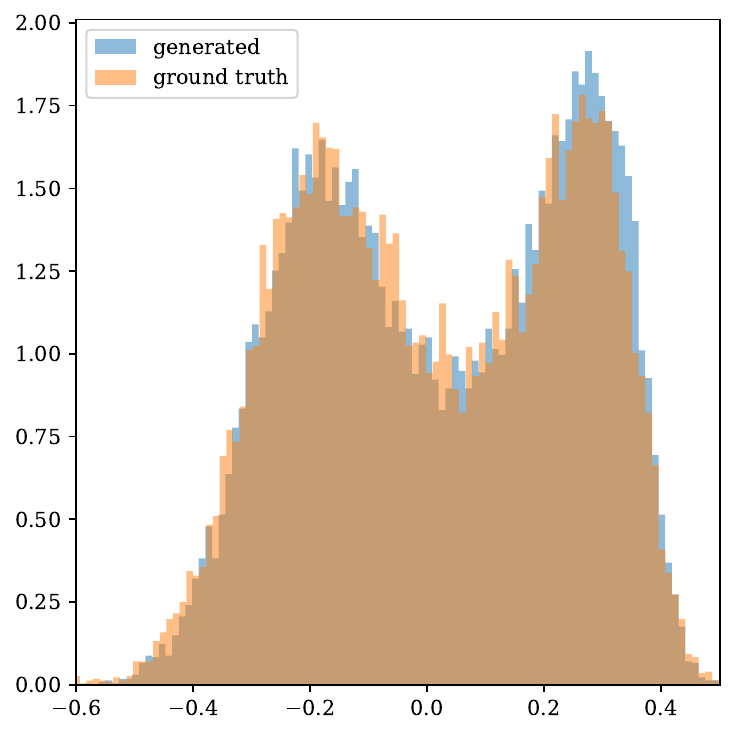}
    \caption{$\mathbb{R}^2$-3P, NETS-P}\label{fig:angle_full}
    \end{subfigure}\hfill
    \begin{subfigure}[b]{0.23\textwidth}
    \centering
    \includegraphics[width=0.9\linewidth]{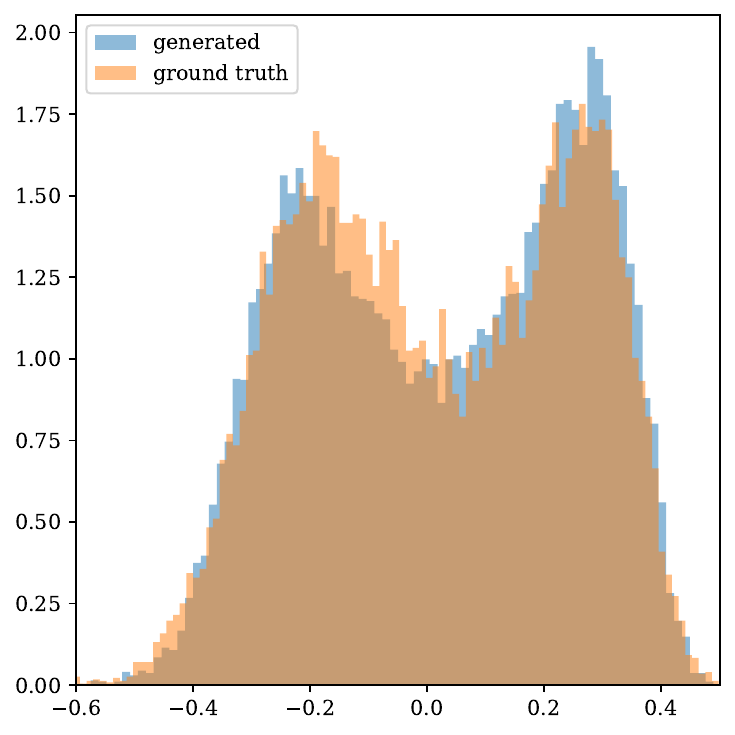}
    \caption{$\mathbb{R}^2$-3P, FES-FM}\label{fig:angle_red}
    \end{subfigure}\hfill
    \begin{subfigure}[b]{0.23\textwidth}
    \centering
    \includegraphics[width=0.9\linewidth]{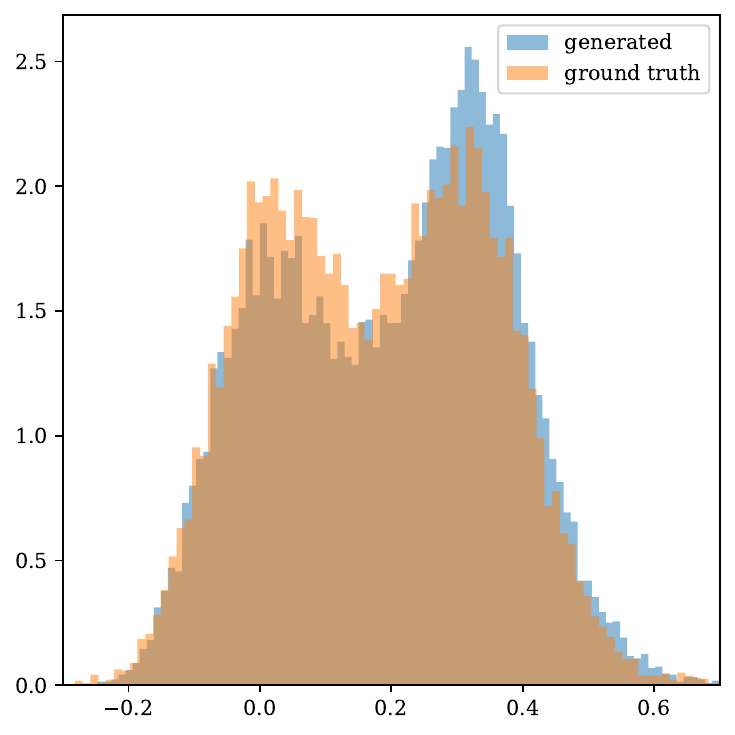}
    \caption{$\mathbb{R}^3$-4P, NETS-P}\label{fig:dihe_full}
    \end{subfigure}\hfill
    \begin{subfigure}[b]{0.23\textwidth}
    \centering
    \includegraphics[width=0.9\linewidth]{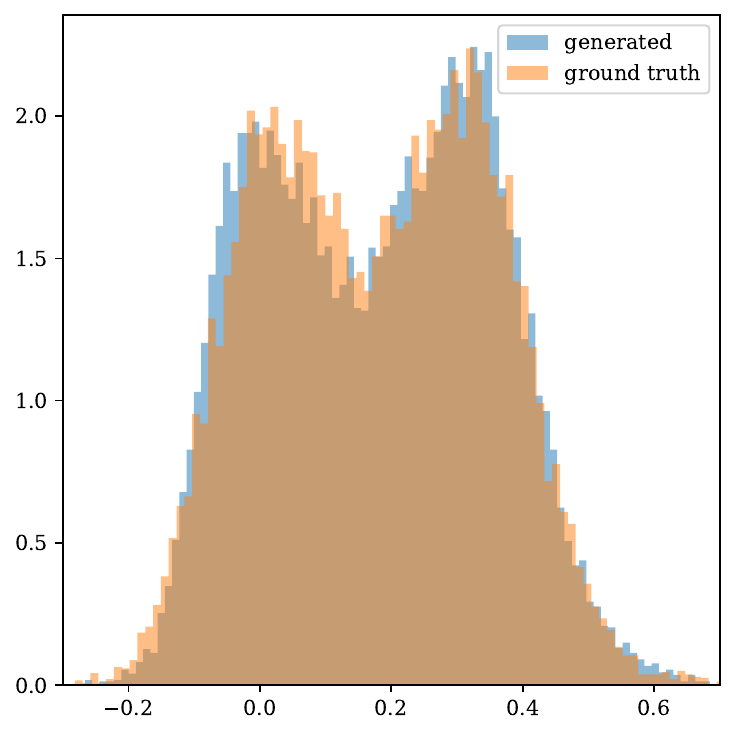}
    \caption{$\mathbb{R}^3$-4P, FES-FM}\label{fig:dihe_red}
    \end{subfigure}
    \caption{Results of the many-particle systems. The red histograms in all subfigures denote the ground-truth samples from the target distribution, generated via Langevin dynamics. (a) Blue histogram: the distribution of CV samples obtained by projecting NETS samples in $\mathbb{R}^2$-3P experiment. (b) Blue histogram: the distribution generated by our FES-FM in $\mathbb{R}^2$-3P experiment. (c) Blue histogram: the distribution of CV samples obtained by projecting NETS samples in $\mathbb{R}^3$-4P experiment. (d) Blue histogram: the distribution generated by our FES-FM in $\mathbb{R}^3$-4P experiment.}\label{fig:many_particle}
    \end{figure}

\begin{table}[ht]
\centering
\caption{Results for many-particle systems. The definitions of \textbf{Time}, \textbf{Error} and \textbf{ACC/Time} are given in Table \ref{table:DW_res}. $\mathbb{R}^2$-3P and $\mathbb{R}^3$-4P denote the three-particle system in $\mathbb{R}^2$ and the four-particle system in $\mathbb{R}^3$, respectively.}\label{table:MP_res}
\begin{tabular}{llcccc}
\toprule
& Method & Time $\downarrow$ & Error $\downarrow$ & ACC/Time $\uparrow$ \\ 
\midrule
\multirow{2}{*}{$\mathbb{R}^2$-3P} & NETS-P & 2.09e+0{\scriptsize $\pm$1.93e-2} & 2.47e-2{\scriptsize $\pm$2.96e-2} & 4.41e+1{\scriptsize $\pm$2.50e+1} \\ 
& FES-FM   &  \cellcolor{gray!15} 3.85e-1{\scriptsize $\pm$1.16e-2} &  \cellcolor{gray!15} 2.08e-2{\scriptsize $\pm$4.20e-3} &  \cellcolor{gray!15} 1.31e+2{\scriptsize $\pm$2.84e+1} \\ 
\midrule
\multirow{2}{*}{$\mathbb{R}^3$-4P} & NETS-P & 2.34e+0{\scriptsize $\pm$1.43e-2} & 1.06e-2{\scriptsize $\pm$4.88e-3} & 4.93e+1{\scriptsize $\pm$2.10e+1} \\ 
& FES-FM &  \cellcolor{gray!15} 3.91e-1{\scriptsize $\pm$1.26e-2} &  \cellcolor{gray!15} 8.98e-3{\scriptsize $\pm$4.54e-3} &  \cellcolor{gray!15} 3.57e+2{\scriptsize $\pm$1.55e+2} \\ 
\bottomrule
\end{tabular}
\end{table}

\subsubsection{Three-particle system in \texorpdfstring{$\mathbb{R}^2$}{R2}}

We consider a system of three particles with positions $x_i\in\mathbb{R}^2$ for $i=1,2,3$, and denote the configuration by $x =\operatorname{Vec}(x_1,x_2,x_3)= (x_1^T,x_2^T,x_3^T)^T$. The potential energy is given by
\begin{align}\label{R2_3P_pot}
U_{\mathrm{target}}(x) = & \alpha_{1}( \|x_1-x_2\|-r_{1})^2 + \alpha_{2}( \|x_1-x_3\|-r_{2})^2 \notag \\
& + \alpha_{3}( \|x_2-x_3\|-r_{3})^2( \|x_2-x_3\|-r_{4})^2.
\end{align}
We select the CV as the cosine of the angle at particle 1 formed by particles 2--1--3, i.e. $\xi(x) =\frac{\langle x_2-x_1, x_3-x_1\rangle}{\|x_2-x_1\|\cdot\|x_3-x_1\|}$. As shown in Figure \ref{fig:R2_3P_demo}, the system exhibits two metastable states, corresponding to different folding conformations.

Table \ref{table:MP_res} shows that our method improves sampling speed while maintaining high accuracy.
Its performance advantage may be attributed to the fact that its training essentially solves a low-dimensional PDE \eqref{PINN_red_ori} (although high-dimensional samples are employed), while NETS addresses a high-dimensional PDE \eqref{PINN_full_point}.
Figure \ref{fig:angle_red} shows that the generated samples match well with the target distribution.

\subsubsection{Four-particle system in \texorpdfstring{$\mathbb{R}^3$}{R3}}

We consider a system of four particles in $\mathbb{R}^3$ with positions $x_i\in\mathbb{R}^3$ for $i=1,2,3,4$, and denote the configuration by $x =\operatorname{Vec}(x_1,x_2,x_3,x_4)= (x_1^T,x_2^T,x_3^T,x_4^T)^T$. The potential energy is defined as
\begin{align}\label{R3_4P_pot}
U_{\mathrm{target}}(x) = & \alpha_{1}( \|x_1-x_2\|-r_{1})^2 +\alpha_{2}( \|x_2-x_3\|-r_{2})^2 + \alpha_{3}( \|x_3-x_4\|-r_{3})^2 \notag \\
&+ \alpha_{4}( \|x_1-x_3\|-r_{4})^2   + \alpha_{5}( \|x_2-x_4\|-r_{5})^2 \notag \\
&+ \alpha_{6}( \|x_1-x_4\|-r_{6})^2( \|x_1-x_4\|-r_{7})^2.
\end{align}
We choose the CV as the cosine of the dihedral angle between the planes spanned by $(x_1,x_2,x_3)$ and $(x_2,x_3,x_4)$. Define the plane normals $n_1=(x_2-x_1)\times (x_3-x_2)$ and $n_2=(x_3-x_2)\times (x_4-x_3)$. Then $\xi(x)=\frac{\langle n_1,n_2\rangle}{\|n_1\|\cdot\|n_2\|}$. As shown in Figure \ref{fig:R3_4P_demo}, the system exhibits two metastable states, corresponding to different folding conformations.

Table \ref{table:MP_res} shows that FES-FM samples faster than the baseline while maintaining, and in this case improving, accuracy.
The overlap between the blue and red histograms in Figure \ref{fig:dihe_red} demonstrates good agreement between the generated CV distribution and the target distribution.

\subsection{Alanine dipeptide in implicit solvent}\label{Sec:alanine}

We further consider alanine dipeptide in implicit solvent as a molecular benchmark for FES-FM. The collective variable is chosen as the backbone dihedral angle $\psi$, which is defined by atoms whose 1-based indices are 7, 9, 15, and 17, and the goal is to learn a reduced transport in this one-dimensional CV space. In this experiment, the prior distribution is taken to be the $\psi$-marginal distribution induced by the ensemble at $800\,\mathrm{K}$, while the target distribution is the corresponding $\psi$-marginal distribution at $300\,\mathrm{K}$. Molecular dynamics simulation at $800\,\mathrm{K}$ is comparatively easy to perform because the system crosses conformational barriers more frequently, whereas at $300\,\mathrm{K}$ the dynamics exhibits pronounced rare-event behavior. This setting tests whether the learned reduced dynamics can transform a high-temperature CV distribution, which explores conformational states more broadly, into the lower-temperature target distribution.

Because the prior and target distributions are induced by the same molecular potential at different temperatures, their corresponding dimensionless potentials differ only through the inverse-temperature factor. Therefore, the path potential used by FES-FM, which linearly interpolates between the prior and target potentials, can also be interpreted as the same molecular potential evaluated at an effective intermediate temperature. More precisely, if $U_T(x)=\beta_T V(x)$ denotes the temperature-scaled molecular potential, then
\begin{equation}
U(x,t)=(1-t)U_{800\,\mathrm{K}}(x)+tU_{300\,\mathrm{K}}(x)
=\beta(t)V(x),
\quad
\beta(t)=(1-t)\beta_{800\,\mathrm{K}}+t\beta_{300\,\mathrm{K}}.
\end{equation}
Thus, the intermediate distributions along the ODE dynamics correspond to CV distributions associated with effective temperatures between $800\,\mathrm{K}$ and $300\,\mathrm{K}$.

Unlike the experiments above, this alanine dipeptide experiment does not use the warm-up process in Algorithm~\ref{algo_warmup}. Equivalently, we set $N_{\mathrm{pre}}=0$ and omit the pretraining of the full-space drift $b_{\theta_0}$. In the non-equilibrium dynamics \eqref{non_equil_X}--\eqref{non_equil_A}, the fixed drift is taken to be the zero field, i.e., $\hat{b}\equiv 0$. To reduce the variance of simulating this coupled system, we use the resampling strategy discussed in Section~\ref{sec:expectation_estimation}, following \citep[Algorithm~1]{tan2025scalable}. The reduced model is then trained directly for the temperature-transfer task from the $800\,\mathrm{K}$ prior distribution to the $300\,\mathrm{K}$ target distribution. As shown in Figure~\ref{fig:alanine_ode_cv}, the generated CV distributions agree well with the molecular-dynamics references along the full interpolation path. The learned free energy surface corresponding to the $300\,\mathrm{K}$ target distribution is shown in Figure~\ref{fig:alanine_fes_300K}.

\begin{figure}[ht]
\centering
\includegraphics[width=0.95\linewidth]{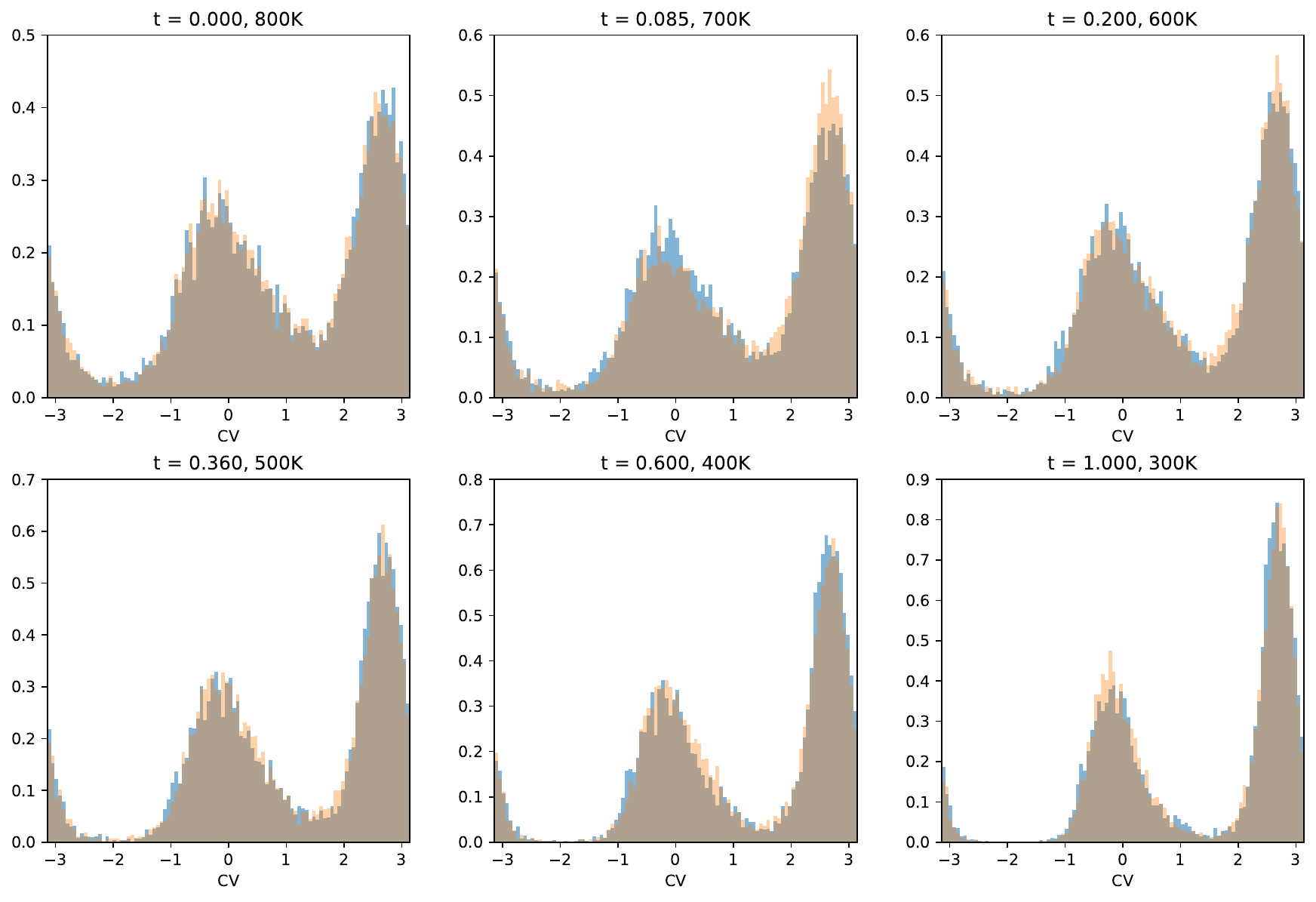}
\caption{CV distributions along the ODE path for alanine dipeptide in implicit solvent. Red curves denote reference distributions obtained from molecular dynamics simulations, and blue curves denote distributions generated by FES-FM. The close agreement at all shown times indicates that FES-FM accurately tracks the path distributions, including the intermediate distributions corresponding to different effective temperatures.}\label{fig:alanine_ode_cv}
\end{figure}

\subsection{Free energy surface computation}\label{sec:fes_cal_results}

We also evaluate the ability of FES-FM to compute the free energy surface using the gradient-matching objective \eqref{loss_fes_cal}. Figure~\ref{fig:FES_cal} reports the reconstructed free energy surfaces for five representative benchmarks: the Müller-Brown potential, the DW-50D example, the three-particle system in $\mathbb{R}^2$, the four-particle system in $\mathbb{R}^3$, and alanine dipeptide in implicit solvent at $300\,\mathrm{K}$. The reference curves are obtained by projecting samples from long-time Langevin dynamics or molecular dynamics onto the corresponding CV space and then applying density estimation. Across these examples, the free energy surfaces learned from \eqref{loss_fes_cal} closely match the reference solutions, indicating that the same reduced framework can be used not only for direct CV sampling but also for recovering the underlying free energy profile. For alanine dipeptide, the learned surface in Figure~\ref{fig:alanine_fes_300K} is less accurate in the high-energy region near $\psi=-2$, where essentially no samples are available. Improving FES reconstruction in such poorly sampled regions is currently under investigation, and the results will be reported elsewhere.

\section{Conclusion}
We propose FES-FM, a reduced flow-matching approach that learns transport directly in CV space, avoiding full-space simulation at generation time. We also introduce an $\mathrm{E}(3)$-invariant Hessian-informed harmonic prior for many-particle systems that produces physically meaningful configurations. Across several benchmark potentials, FES-FM significantly improves sampling speed while maintaining accuracy compared with full-space baselines.

\section*{Acknowledgements}
Tiejun Li acknowledges the support from National Key R\&D Program of China under grant 2021YFA1003301, and National Science Foundation of China under grant 12288101. This work is supported by High-performance Computing Platform of Peking University.

\begin{figure}[h]
\centering
\begin{subfigure}[b]{0.30\textwidth}
\centering
\includegraphics[width=0.95\linewidth]{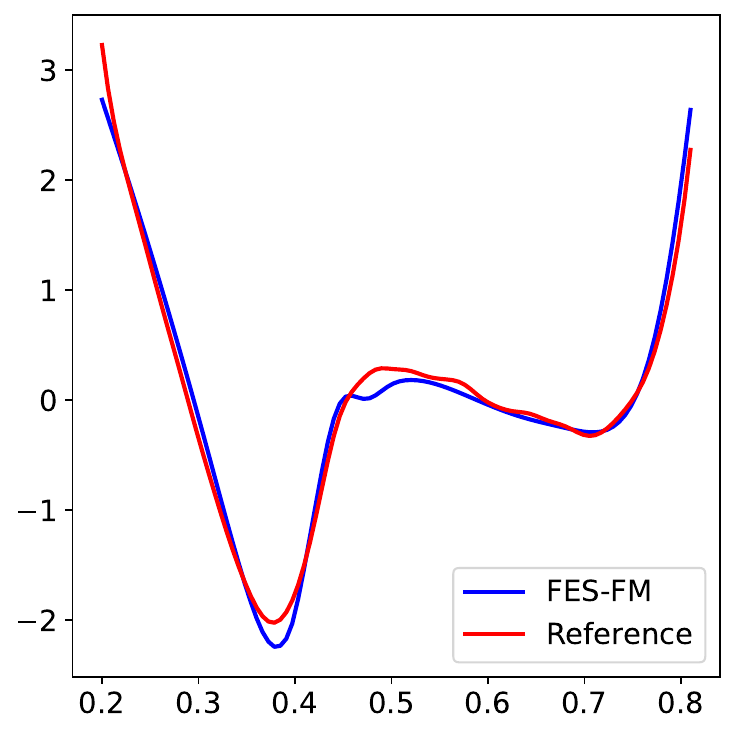}
\caption{Müller-Brown}\label{fig:MB_fes}
\end{subfigure}
\hspace{0.06\textwidth}
\begin{subfigure}[b]{0.30\textwidth}
\centering
\includegraphics[width=0.95\linewidth]{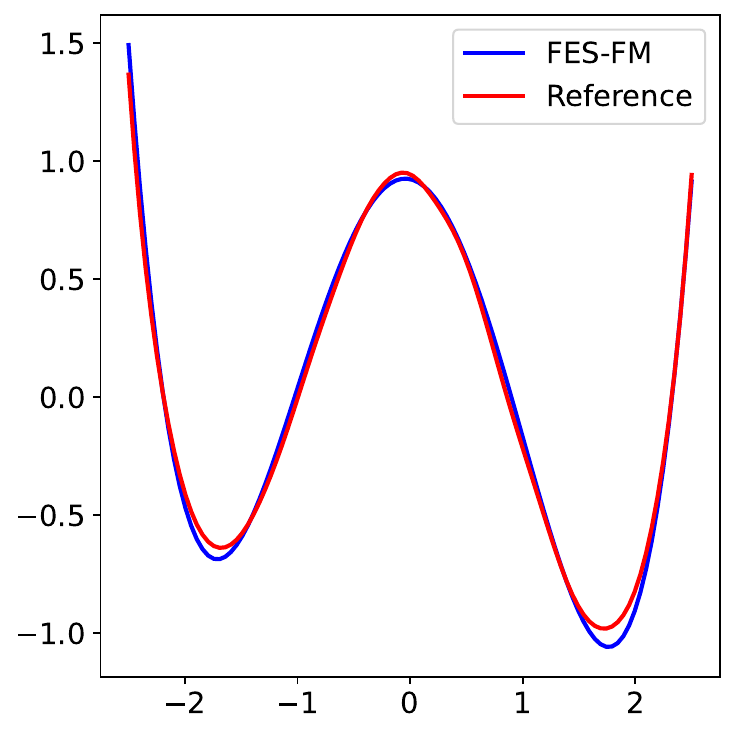}
\caption{DW-50D}\label{fig:DW50_fes}
\end{subfigure}\\[0.8em]
\begin{subfigure}[b]{0.30\textwidth}
\centering
\includegraphics[width=0.95\linewidth]{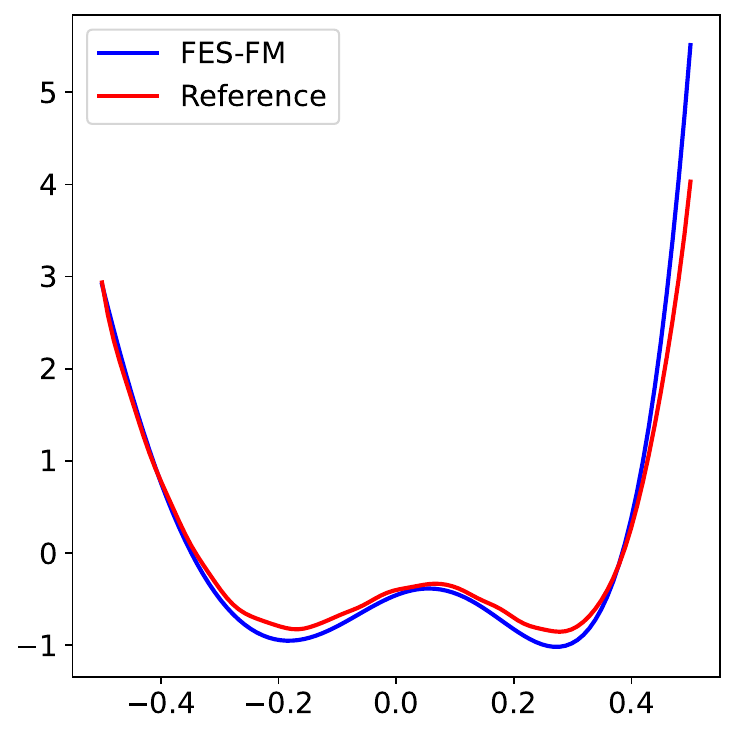}
\caption{$\mathbb{R}^2$-3P}\label{fig:angle_fes}
\end{subfigure}
\hspace{0.025\textwidth}
\begin{subfigure}[b]{0.30\textwidth}
\centering
\includegraphics[width=0.95\linewidth]{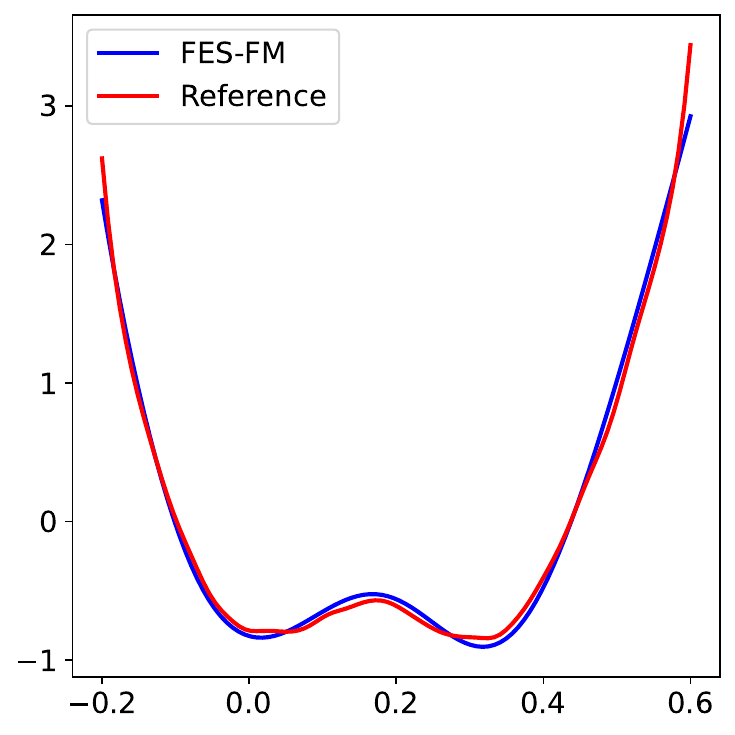}
\caption{$\mathbb{R}^3$-4P}\label{fig:dihe_fes}
\end{subfigure}
\hspace{0.025\textwidth}
\begin{subfigure}[b]{0.30\textwidth}
\centering
\includegraphics[width=0.95\linewidth]{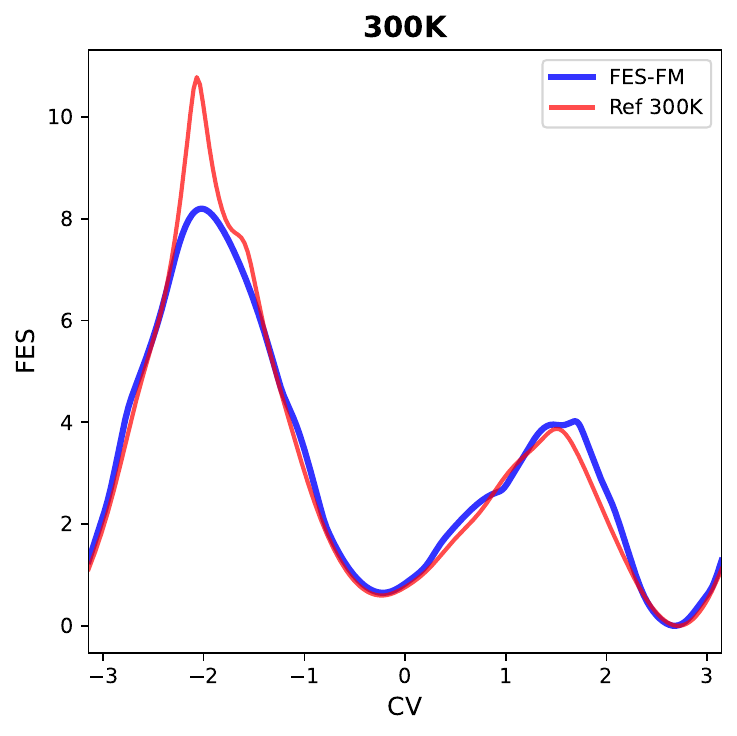}
\caption{Alanine dipeptide}\label{fig:alanine_fes_300K}
\end{subfigure}
\caption{Free energy surface computation results. Red curves denote reference solutions obtained by projecting samples from long-time Langevin dynamics or molecular dynamics onto the CV space and applying density estimation. Blue curves denote the free energy surfaces trained with the loss \eqref{loss_fes_cal}. Panel (e) shows the learned free energy surface for alanine dipeptide in implicit solvent at $300\,\mathrm{K}$.}\label{fig:FES_cal}
\end{figure}
\clearpage

\appendix

\section{A Comparison List between NETS and FES-FM}\label{compare_summary}

To better illustrate the connection and difference between NETS (full model) and FES-FM (reduced model), we provide a detailed comparison of the related concepts and formulas in Table~\ref{model_summary_combined}.

\begin{table}[ht]
\centering
\caption{Model summary and comparison between the full NETS model and the reduced FES-FM model}\label{model_summary_combined}
\begin{tabular}{lc}
\toprule
\multicolumn{2}{c}{\textbf{NETS (full model) in $\mathbb{R}^n$}} \\ 
\midrule
Target density & $p(x) = \frac{1}{Z} e^{-U_{\mathrm{target}}(x)}$ \\
Target potential & $U_{\mathrm{target}}(x)$ \\
\midrule
Marginal density & $p(x,t) = \frac{ e^{-U(x,t)}}{Z(t)}$, $Z(t) = \int_{\mathbb{R}^{n}} e^{-U(x,t)} \mathrm{d}x$ \\
Marginal potential & $U(x,t) = (1 - t) U_{0}(x) + t U_{\mathrm{target}}(x)$ \\
Marginal measure & $\mathrm{d} \nu_{t}(x) =p(x, t) \mathrm{d}x$ \\
\midrule
Transport map & $\frac{\mathrm{d}}{\mathrm{d}t} X_t= b(X_t, t)$ \\
Transport equation & $\partial_t U(x,t)  + b(x,t)\cdot \nabla U(x,t) - \nabla \cdot b(x,t) + \partial_t \log Z(t)=0$ \\
\bottomrule
\\
\multicolumn{2}{c}{\textbf{Model Reduction $\xi:\mathbb{R}^n \to \mathbb{R}^d$}}\\
\\
\toprule
\multicolumn{2}{c}{\textbf{FES-FM (reduced model) in $\mathbb{R}^d$}} \\ 
\midrule
Target density & $\rho(y)= \int_{\Sigma_{y}} p(x) |\nabla \xi(x)^{T} \nabla \xi(x)|^{-\frac{1}{2}} \mathrm{d} \sigma_{\Sigma_{y}}(x).$ \\
Target potential & $F(y) = -\log \rho(y)$ \\
\midrule
Marginal density & $\rho(y,t) =\int_{\Sigma_{y}} p(x,t) |\nabla \xi(x)^{T} \nabla \xi(x)|^{-\frac{1}{2}} \mathrm{d} \sigma_{\Sigma_{y}}(x)$ \\
Marginal potential & $F(y,t) = -\log \rho(y,t)$  \\
Marginal measure & $\mathrm{d} \mu_{t}(y)  =\rho(y,t) \mathrm{d}y$ \\
\midrule
Transport map & $\frac{\mathrm{d}}{\mathrm{d}t} Y_t= u(Y_t, t)$  \\
Transport equation & $\mathbb{E}_{\mu_{\Sigma_{y},t}} \partial_{t} U(x,t) +u(y,t) \cdot \mathbb{E}_{\mu_{\Sigma_{y},t}} D(x,t)  - \nabla \cdot u(y,t) + \partial_{t} \log Z(t)=0$ \\
\bottomrule
\end{tabular}
\end{table}

\section{Details of the Hessian-informed harmonic prior distribution}\label{Hessian_Prior}

Define $e_i\in\mathbb{R}^3$ as the $i$-th standard basis vector. Let $E_{i,j}\in\mathbb{R}^{3\times3}$ have a 1 at entry $(i,j)$ and zeros elsewhere, and set $J_{i,j}=E_{i,j}-E_{j,i}$. Let ${\bf 1}_M=(1,\dots,1)^T\in\mathbb{R}^M$ and define $b_i={\bf 1}_M\otimes e_i\in\mathbb{R}^{3M}$. We say a function $U$ is $\mathrm{O}(3)$-invariant if $U(x) = U((I_M\otimes R)x)$ for any $R\in \mathrm{O}(3)$. We say $U$ is translation-invariant if $U(x) = U(x+b_i)$ for all $1\leq i\leq 3$. The degrees of freedom corresponding to translation are 3, and those corresponding to rotation are also 3. The following theorem characterizes the null space of $H(x_0)$.

\begin{theorem}\label{Thm_null_space_hessian}
Assume $U$ is twice continuously differentiable, $\mathrm{O}(3)$-invariant, and translation-invariant. Let $H(x_0) = \nabla^2 U(x_0)$ denote the Hessian at $x_0\in\mathbb{R}^{3M}$. Then, the null space of $H(x_0)$ is spanned by the vectors $(I_M \otimes J_{i,j})x_0$ ($1\leq i<j \leq 3$) together with $b_i$ ($1\leq i \leq 3$) and consequently $\operatorname{rank}(H(x_0)) = n - 6$.
\end{theorem}

\begin{proof}[Proof of Theorem \ref{Thm_null_space_hessian}]

By translation invariance, $U(x+tb_i)=U(x)$ for all $t\in\mathbb{R}$ and $1\leq i\leq3$. Differentiating twice with respect to $t$ and evaluating at $t=0$ (via chain rule):
\begin{equation}
0 = \frac{\mathrm{d}^2}{\mathrm{d}t^2}U(x+tb_i)\bigg|_{t=0} = b_i^T H(x) b_i.
\end{equation}
At the local minimum $x_0$ ($\nabla U(x_0)=0$), this simplifies to $H(x_0)b_i=0$ for $1\leq i\leq3$, so $b_i\in\ker(H(x_0))$.

By rotation invariance, we have $\nabla U((I_M\otimes R)x_0)=(I_M\otimes R)\nabla U(x_0)=0$, for $R\in\mathrm{O}(3)$. For the skew-symmetric $J_{i,j}$, $e^{tJ_{i,j}}\in\mathrm{O}(3)$, so differentiating $\nabla U((I_M\otimes e^{tJ_{i,j}})x_0)=0$ with respect to $t$ at $t=0$ gives:
\begin{equation}
H(x_0)(I_M\otimes J_{i,j})x_0=0, \quad \forall 1\leq i<j\leq3,
\end{equation}
meaning $(I_M\otimes J_{i,j})x_0\in\ker(H(x_0))$.

Under non-degenerate conditions, the 3 translation vectors $\{b_i\}$ and 3 rotation vectors $\{(I_M\otimes J_{i,j})x_0\}$ are linearly independent; thus, $H(x_0)$ has 6 trivial degrees of freedom.
\end{proof}

\begin{theorem}\label{thm_invariant_prior_pot}
The potential $U_0$ defined in \eqref{prior_potential} is $\mathrm{O}(3)$-invariant.
\end{theorem}

\begin{proof}[Proof of Theorem \ref{thm_invariant_prior_pot}]
For any $Q\in\mathrm{O}(3)$, $R_0^{*}((I_{M}\otimes Q)x)=R_0^{*}(x)Q^T$. Therefore, 
\begin{equation}
(I_{M}\otimes R_0^{*}((I_{M}\otimes Q)x)) (I_{M}\otimes Q)x=(I_{M}\otimes R_0^{*}(x))(I_{M}\otimes Q^T)(I_{M}\otimes Q)x=(I_{M}\otimes R_0^{*}(x))x.
\end{equation}
Then, we have
\begin{align}
& U_0((I_{M}\otimes Q)x) \notag \\
= &\frac{1}{2} \bigg( (I_{M}\otimes R_0^{*}((I_{M}\otimes Q)x))(I_{M}\otimes Q)x -  x_0 \bigg)^T H (x_0)  \notag \\
&  \hspace{2cm} \cdot \bigg( (I_{M}\otimes R_0^{*}((I_{M}\otimes Q)x)) (I_{M}\otimes Q)x- x_0 \bigg) \notag \\
=&\frac{1}{2} \big( (I_{M}\otimes R_0^{*}(x))x -  x_0 \big)^T H (x_0) \big( (I_{M}\otimes R_0^{*}(x)) x- x_0 \big) \notag \\
=& U_0(x).
\end{align}
This completes the proof.
\end{proof}

\section{Experimental details}

All experiments are conducted on a single NVIDIA A100-PCIe-40GB GPU.
We use $K=100$ time steps to sample the non-equilibrium state in all experiments except alanine dipeptide, for which $K=200$ time steps are used.
The time-dependent diffusion coefficient $\epsilon_t$ in \eqref{non_equil_X} is set to a constant, i.e. $\epsilon_t\equiv\epsilon$.
For the loss function in \eqref{loss_u_ori}, we fix the parameters as $\lambda = 1.0$ for all experiments.
We use multilayer perceptrons (MLPs) with SiLU activations. Both $c_{\theta_0}$ and $c_{\theta_1}$ have three hidden layers with 64 units per layer. Models are trained using PyTorch with the Adam optimizer. During training, the learning rate is linearly decayed from 0.001 to 0.0005, and gradients are clipped when their $\ell_2$-norm exceeds a predefined threshold of 10. The batch size is chosen as 2048.
The ODEs \eqref{ode_full} and \eqref{ode_red} are solved via the Euler method with a step size of 0.001.
The 1-Wasserstein distance is computed using 10,000 points sampled from the generated distribution and 10,000 points sampled from the reference distribution.
Remaining hyperparameters are summarized in Table \ref{Hyperparameters}. Detailed experimental setups are provided in the subsections below.

\begin{table}[ht]
\centering
\caption{Parameters in our experiments. $\epsilon$ denotes the time-dependent diffusion coefficient in \eqref{non_equil_X}. $N_{\mathrm{pre}}$ is the number of warm-up epochs in Algorithm~\ref{algo_warmup}, and $N_{\mathrm{epoch}}$ is the number of training epochs in Algorithm~\ref{algo_train}. $N_{\mathrm{n}}^{b}$ and $N_{\mathrm{l}}^{b}$ are the numbers of hidden nodes per layer and hidden layers in the neural network $b_{\theta_0}$, respectively. $N_{\mathrm{n}}^r$ and $N_{\mathrm{l}}^r$ are the numbers of hidden nodes per layer and hidden layers in the neural networks $u_{\theta_1}$ and $v_{\theta_1}$, respectively. DW-50D, DW-100D and DW-200D denote experiments on the double-well potential considered in Section \ref{Sec:DW}.}\label{Hyperparameters}
\begin{tabular}{lcccccccc}
\toprule
Datasets & $\epsilon$ & $N_{\mathrm{pre}}$  & $N_{\mathrm{epoch}}$  & $N_{\mathrm{n}}^{b}$ & $N_{\mathrm{l}}^{b}$ & $N_{\mathrm{n}}^r$& $N_{\mathrm{l}}^r$ \\
\midrule
Müller-Brown & 0.2 & 2000 & 20000  & 128 & 3 & 64 & 3 \\
\midrule
DW-50D & 0.1& 2000 & 5000 & 512 & 3 & 64 & 2 \\
DW-100D & 0.1& 2000 & 5000 & 512 & 3 & 64 & 2 \\
DW-200D & 0.1& 2000 & 5000 & 512 & 3 & 64 & 2 \\
\midrule
$\mathbb{R}^2$-3P & 0.02& 2000 & 5000  & 256 & 3  & 64 & 3 \\
$\mathbb{R}^3$-4P & 0.02& 2000 & 5000 & 256 & 3 & 64 & 3  \\
\midrule
Dipeptide & 0.0001 & 0 & 1000  & 256 & 3  & 64 & 3 \\
\bottomrule
\end{tabular}
\end{table}

\subsection{Computational cost}\label{sec:train_cost}

Table~\ref{tab:FES-FM-traincost} reports the per-epoch computational cost of the main training components. We define $T^{\mathrm{train}}_{\mathrm{sim}}$ as the wall-clock time for generating one batch of training data by simulating the non-equilibrium dynamics in \eqref{non_equil_X}--\eqref{non_equil_A}. The quantities $T^{\mathrm{train}}_{\mathrm{NETS}}$ and $T^{\mathrm{train}}_{\mathrm{FES-FM}}$ denote the per-epoch optimization time of the NETS full-space training step and the FES-FM reduced training step, respectively. For relatively simple examples, such as the Müller-Brown and DW-50D potentials, these two optimization costs are comparable. For more complex examples, such as $\mathbb{R}^2$-3P and $\mathbb{R}^3$-4P, $T^{\mathrm{train}}_{\mathrm{NETS}}$ is substantially larger than $T^{\mathrm{train}}_{\mathrm{FES-FM}}$, showing that the reduced FES-FM update is cheaper than the corresponding full-space NETS update. The warm-up process therefore incurs an additional training-stage cost only when it is used; at inference time, FES-FM samples by solving the reduced ODE \eqref{ode_red}, so no warm-up trajectories are generated during sampling.

\begin{table}[ht]
\centering
\caption{Per-epoch computational cost during training, measured in seconds.}\label{tab:FES-FM-traincost}
\begin{tabular}{lcccc}
\toprule
 & Müller-Brown  & DW-50D & $\mathbb{R}^2$-3P & $\mathbb{R}^3$-4P \\
\midrule
$T^{\mathrm{train}}_{\mathrm{sim}}$  & 0.417 & 0.286 & 0.735  & 1.098 \\
$T^{\mathrm{train}}_{\mathrm{NETS}}$ & 0.010 & 0.010 & 0.141  & 0.247 \\
$T^{\mathrm{train}}_{\mathrm{FES-FM}}$& 0.010 & 0.009 & 0.008  & 0.009 \\
\bottomrule
\end{tabular}
\end{table}

\subsection{Warm-up process}\label{sec:warmup_process}

Equation \eqref{Jar_reweight} gives an exact reweighting identity for estimating expectations under $\nu_t$, but the variance of the estimator depends on the drift used in the non-equilibrium dynamics. An arbitrary drift may produce a large mismatch between the marginal law of $X_t^{\hat{b}}$ and the desired interpolated distribution $p(x,t)$, leading to high-variance importance weights. To reduce such variance, in experiments where warm-up is used, we follow the NETS full-space training procedure and first train a full-space drift $b_{\theta_0}$ with the loss \eqref{b_loss}. In FES-FM, this NETS-style warm-up stage is used only to obtain an approximately useful drift for reweighting; it is not necessary to fully solve the full-space transport problem. After warm-up, the resulting $b_{\theta_0}$ is fixed and used as $\hat{b}$ in Algorithm~\ref{algo_train}. The alanine dipeptide experiment in Section~\ref{Sec:alanine} skips this warm-up process.

\begin{algorithm}[ht]
\caption{Warm-up process (NETS)}\label{algo_warmup}
\begin{algorithmic}[1]
\STATE \textbf{Initialize:} neural networks $b_{\theta_0}$, $c_{\theta_0}$; warm-up epochs $N_{\mathrm{pre}}$; time steps $K$
\FOR{epoch=$1, \dots, N_{\mathrm{pre}}$}
    \STATE Generate trajectories $\{(X_{t_k},A_{t_k})\}_{1\leq k\leq K}$ by solving \eqref{non_equil_X}--\eqref{non_equil_A} with the full-space drift $b_{\theta_0}$
    \STATE Calculate the NETS full-space loss $\mathcal{L}_0[b_{\theta_0}, c_{\theta_0}]$ in \eqref{b_loss}, where the expectation $\mathbb{E}_{\nu_{t}}[\cdot]$ is estimated using \eqref{Jar_reweight} with $\{(X_{t_k},A_{t_k})\}$
    \STATE Update the parameters of $b_{\theta_0}, c_{\theta_0}$ by gradient descent on $\mathcal{L}_0$
\ENDFOR
\STATE \textbf{Return:} the pretrained drift $b_{\theta_0}$
\end{algorithmic}
\end{algorithm}

\subsection{Müller-Brown potential}


The Müller-Brown surface is defined as

\begin{align}
U_{\mathrm{target}}(x_1, x_2)
=&-200\exp\big(-(x_1-1)^2-10x_2^2\big) -100\exp\big(-x_1^2-10(x_2-0.5)^2\big) \notag \\
&-170\exp\big(-6.5(x_1+0.5)^2+11(x_1+0.5)(x_2-1.5)-6.5(x_2-1.5)^2\big) \notag \\
&+15\exp\big(0.7(x_1+1)^2+0.6(x_1+1)(x_2-1)+0.7(x_2-1)^2\big).
\end{align}
We choose the prior to be a normal distribution with mean vector $(-0.7, 0.7)^T$ and the isotropic standard deviation $0.3$.

The CV map is constrained to satisfy $ \xi_{\theta}(\phi(s)) = s $, where $ \phi(s) \in \mathbb{R}^2 $ denotes the transition path obtained by the string method \citep{weinan2002string}. Furthermore, we enforce that the gradient of the CV map is parallel to $\dot{\phi}(\xi_{\theta}(x))$. We optimize $\xi_{\theta}$ via the following loss function:
\begin{equation}
\int_0^1 \left|\xi_{\theta}(\phi(s)) - s \right|^2  \mathrm{d} s 
+ \mathbb{E} \big\| \nabla \xi_{\theta}(x) \times \dot{\phi}(\xi_{\theta}(x)) \big\|^2,
\end{equation}
where $\times$ denotes the vector cross product, and the expectation is taken over the target distribution. The neural network representing the CV map is also an MLP with two hidden layers and 128 hidden units per layer.

\subsection{High-dimensional example}

We choose the prior to be the standard normal distribution in $\mathbb{R}^n$. The divergence term $\nabla\cdot b_{\theta_0}(x)$ in \eqref{b_loss} and \eqref{non_equil_A} is estimated using the Hutchinson trace estimator, as done in NETS.

\subsection{Many-particle systems}

The divergence term $\nabla\cdot b_{\theta_0}(x)$ is also estimated using the Hutchinson trace estimator.
To endow the neural network $b_{\theta_0}$ with $\mathrm{E}(3)$-equivariance, we adopt the construction of \citet{Shi2021LearningGF}. Specifically, we parameterize $b_{\theta_0}$ as the gradient of a scalar-valued neural network, i.e. $b_{\theta_0}=\nabla \tilde{b}_{\theta_0}$. Here, $\tilde{b}_{\theta_0}$ takes interatomic distances and time as input. Alternative $\mathrm{E}(3)$-equivariant architectures are also available, such as those utilizing graph neural networks \citep{satorras2021n} and those based on alignment with respect to a given reference configuration \citep{liu2025improving,liu2025riemannian}.

For the three-particle system in $\mathbb{R}^2$, the parameters in \eqref{R2_3P_pot} are chosen as $\alpha_1=5000/49,\alpha_2=5000/49,\alpha_3=50$, $r_1=2,r_2=2,r_3=2.4,r_4=3.1$. For the four-particle system in $\mathbb{R}^3$, the parameters in \eqref{R3_4P_pot} are chosen as $\alpha_1=\alpha_2=\alpha_3=\alpha_4=\alpha_5=5000/49,\alpha_6=200$, $r_1=2,r_2=1,r_3=2,r_4=2.236,r_5=2.236,r_6=2.5,r_7=3.0$.

\section{Ablation study}\label{sec:abl}

\begin{figure}[htbp]
\centering
\begin{subfigure}[b]{0.31\textwidth}
\centering
\includegraphics[width=0.98\linewidth]{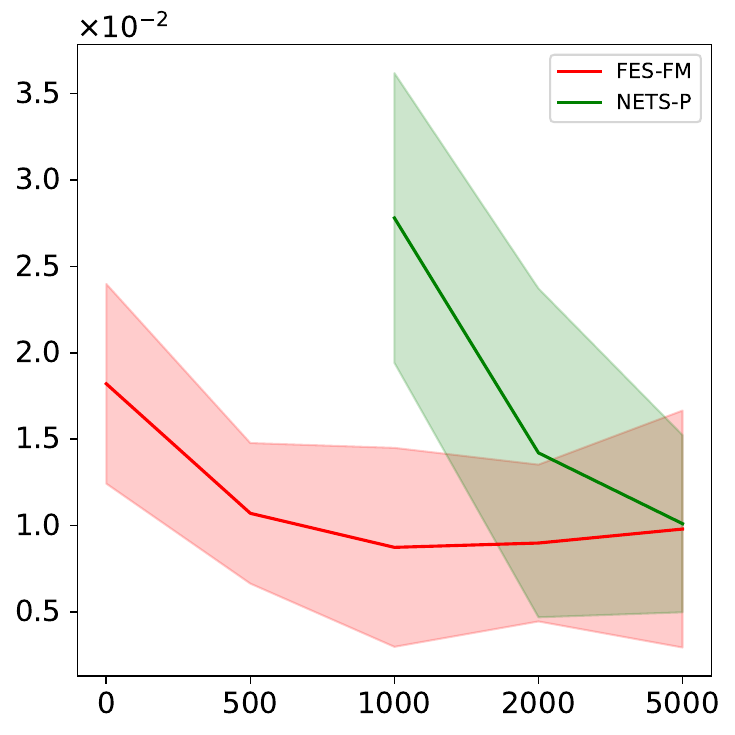}
\caption{$\mathbb{R}^3$-4P, Error vs $N_{\mathrm{pre}}$}\label{fig:abl_dihe_preE}
\end{subfigure}
\begin{subfigure}[b]{0.31\textwidth}
\centering
\includegraphics[width=0.98\linewidth]{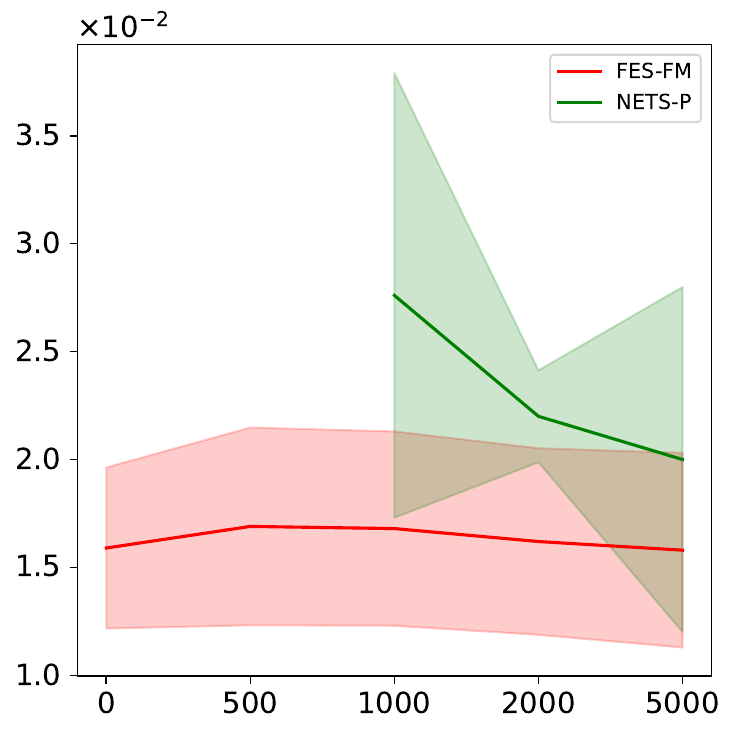}
\caption{DW-50D, Error vs $N_{\mathrm{pre}}$}\label{fig:abl_DW50_preE}
\end{subfigure}
\begin{subfigure}[b]{0.31\textwidth}
\centering
\includegraphics[width=0.98\linewidth]{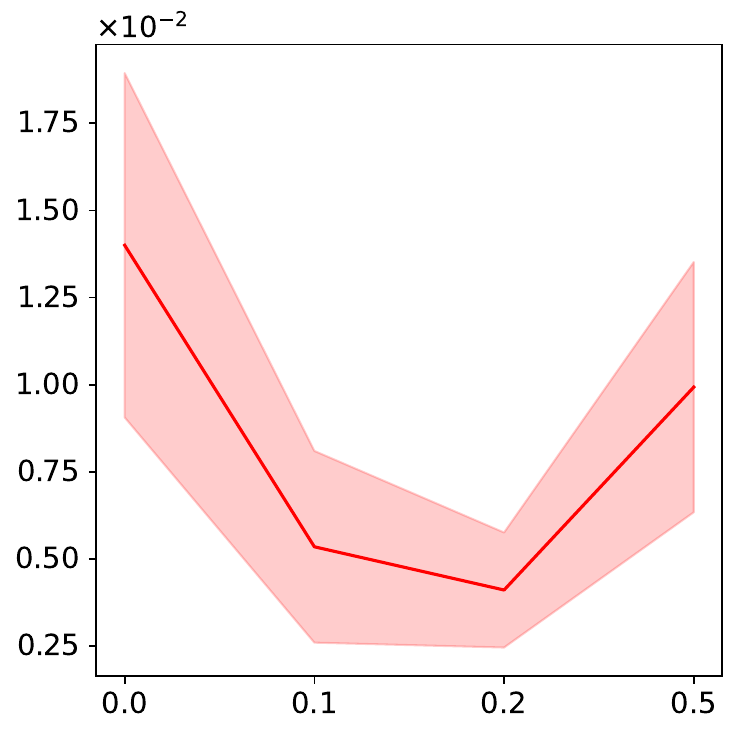}
\caption{Müller-Brown, Error vs $\epsilon$}\label{fig:abl_MB_eps}
\end{subfigure}
\caption{Ablation study. Solid lines show the mean; shaded regions indicate one standard deviation. (a) Error versus $N_{\mathrm{pre}}$ in the $\mathbb{R}^3$-4P experiment: NETS-P (green) uses $N_{\mathrm{pre}}$ warm-up epochs, and FES-FM (red) uses $N_{\mathrm{pre}}$ training epochs. (b) Error versus $N_{\mathrm{pre}}$ in the DW-50D experiment: NETS-P (green) uses $N_{\mathrm{pre}}$ warm-up epochs, and FES-FM (red) uses $N_{\mathrm{pre}}$ training epochs. (c) Error versus $\epsilon$ in the Müller-Brown experiment.}
\end{figure}


We perform an ablation study on the number of warm-up epochs, $N_{\mathrm{pre}}$. In the NETS warm-up stage, $N_{\mathrm{pre}}$ controls the accuracy of the pre-trained drift $b_{\theta_0}$, which in turn affects the variance of the Jarzynski reweighting estimator in \eqref{Jar_reweight}. For the $\mathbb{R}^3$-4P experiment (Figure \ref{fig:abl_dihe_preE}), the distribution of $X_t^{\hat{b}}$ can deviate substantially from the interpolated distribution $p(x,t)$. In this case, warm-up pre-training is beneficial and can noticeably improve performance. In particular, we observe suboptimal performance when $N_{\mathrm{pre}}=0$ or 500 (see Figure~\ref{fig:abl_dihe_preE}). For the DW-50D experiment (Figure \ref{fig:abl_DW50_preE}), the discrepancy between $X_t^{\hat{b}}$ and $p(x,t)$ is small; at a minimum, both distributions cover the relevant region of configuration space. Consequently, increasing $N_{\mathrm{pre}}$ provides little additional benefit. Across all settings, our method yields consistently satisfactory results (see Figure~\ref{fig:abl_DW50_preE}).

Moreover, Figure \ref{fig:abl_MB_eps} shows the effect of the diffusion coefficient $\epsilon$ on the error, highlighting that an appropriate choice of $\epsilon$ is crucial for achieving good performance.

\section{Limitations and future work}

Future work could extend the method to realistic molecular mechanics force fields, such as AMBER or CHARMM, where the dimension of the configuration space is much higher and the energy landscape is more complex. Currently, our method still relies on the NETS high-dimensional sampling procedure during training; a promising direction is to bypass this step and develop methods that learn and sample directly from the free energy surface. Coupling FES sampling with the automatic identification of CVs is also an important direction for future studies.

\bibliographystyle{elsarticle-harv} 

\bibliography{ref}

\end{document}